\documentclass{article}

\usepackage{microtype}
\usepackage{graphicx}
\usepackage{subfigure}
\usepackage{booktabs} % for professional tables
\usepackage{amsmath}
\usepackage{amssymb}

\usepackage{hyperref}

\usepackage[accepted]{icml2018}

\usepackage{siunitx}

\newtheorem{theorem}{Theorem}[section]
\newtheorem{lemma}[theorem]{Lemma}
\newtheorem{definition}[theorem]{Definition}

\newtheorem{corollary}[theorem]{Corollary}

\newtheorem{assumptions}[theorem]{Assumption}
\newenvironment{proof}{\par\noindent{\bf Proof:\ }}{\hfill$\Box$\\[2mm]}

\newif\ifpaper
\papertrue 
\def\RR{\mathbb{R}}
\def\>{\rangle}
\def\rank{\operatorname{\textit{rank}}}
\def\vec{\operatorname{\textit{vec}}}

\def\Set#1{\left\{ #1 \right\}}
\def\Bigbar#1{\mathrel{\left|\vphantom{#1}\right.}}%\n@space}}
\def\Setbar#1#2{\Set{#1 \Bigbar{#1 #2} #2}}

\newcommand{\inner}[1]{\left\langle#1\right\rangle}

\newcommand{\norm}[1]{\left\|#1\right\|}
\newcommand{\Fnorm}[1]{\left\|#1\right\|_{F}}
\newcommand{\maxnorm}[1]{\left\|#1\right\|_{\textrm{max}}}
\newcommand{\minnorm}[1]{\left\|#1\right\|_{\textrm{min}}}

\newcommand{\svmax}[1]{\sigma_{\textrm{max}}(#1)}
\newcommand{\svmin}[1]{\sigma_{\textrm{min}}(#1)}
\newcommand{\evmax}[1]{\lambda_{\textrm{max}}(#1)}
\newcommand{\evmin}[1]{\lambda_{\textrm{min}}(#1)}
\def\bydef{\mathrel{\mathop:}=}

\def\tr{\mathop{\rm tr}\nolimits}
\def\det{\mathop{\rm det}\nolimits}

\def\range{\mathop{\rm range}\nolimits}

\newcommand{\Id}{\mathbb{I}}
\def\argmin{\mathop{\rm arg\,min}\limits}%    a math operator.

\def\min{\mathop{\rm min}\nolimits}
\def\max{\mathop{\rm max}\nolimits}
\def\ones{\mathbf{1}}

\def\ie{\textit{i.e. }}
\def\eg{\textit{e.g. }}

\icmltitlerunning{Optimization Landscape and Expressivity of Deep CNNs}

\begin{document}

\twocolumn[
\icmltitle{Optimization Landscape and Expressivity of Deep CNNs}

% It is OKAY to include author information, even for blind
% submissions: the style file will automatically remove it for you
% unless you've provided the [accepted] option to the icml2018
% package.

% List of affiliations: The first argument should be a (short)
% identifier you will use later to specify author affiliations
% Academic affiliations should list Department, University, City, Region, Country
% Industry affiliations should list Company, City, Region, Country

% You can specify symbols, otherwise they are numbered in order.
% Ideally, you should not use this facility. Affiliations will be numbered
% in order of appearance and this is the preferred way.
\icmlsetsymbol{equal}{*}

\begin{icmlauthorlist}
\icmlauthor{Quynh Nguyen}{a}
\icmlauthor{Matthias Hein}{b}
\end{icmlauthorlist}

\icmlaffiliation{a}{Department of Mathematics and Computer Science, Saarland University, Germany}
\icmlaffiliation{b}{University of T{\"u}bingen, Germany}

\icmlcorrespondingauthor{Quynh Nguyen}{quynh@cs.uni-saarland.de}
% \icmlcorrespondingauthor{}{}

% You may provide any keywords that you
% find helpful for describing your paper; these are used to populate
% the "keywords" metadata in the PDF but will not be shown in the document
\icmlkeywords{loss surface, local minima, global optimality, optimization landscape, expressivity, convolutional neural networks}

\vskip 0.3in
]

\printAffiliationsAndNotice{}  % leave blank if no need to mention equal contribution
% \printAffiliationsAndNotice{\icmlEqualContribution} % otherwise use the standard text.

\begin{abstract}
  We analyze the loss landscape and expressiveness of practical deep convolutional neural networks (CNNs) 
  with shared weights and max pooling layers.
  We show that such CNNs produce linearly independent features at a ``wide'' layer which has more neurons than 
  the number of training samples. This condition holds e.g. for the VGG network.
  Furthermore, we provide for such wide CNNs necessary and sufficient conditions for global minima with zero training error. 
  For the case where the wide layer is followed by a fully connected layer
  we show that almost every critical point of the empirical loss is a global minimum with zero training error. 
  Our analysis suggests that both depth and width are very important in deep learning.  
  While depth brings more representational power and allows the network to learn high level features,
  width smoothes the optimization landscape of the loss function in the sense that 
  a sufficiently wide network  has a well-behaved loss surface with almost no bad local minima.
\end{abstract}

\section{Introduction}
It is well known that the optimization problem for training neural networks 
can have exponentially many local minima \citep{Auer96,SafSha2016} 
and NP-hardness has been shown in many cases \citep{Blum1989,Sim2002,LivShaSha2014,Shamir2017,Shwartz2017}. 
However, it has been empirically observed \citep{Dauphin16, Goodfellow15} 
that the training of state-of-the-art deep CNNs \citep{CunEtAl1990,KriSutHin2012}, 
which are often overparameterized, is not hampered by suboptimal local minima.

\begin{table}[t]
\caption{The maximum width of all layers in several state-of-the-art CNN architectures 
compared with the size of ImageNet dataset ($N\approx 1200K$).
All numbers are lower bounds on the true width.}
\label{tab:net_width}
\vskip 0.15in
\begin{center}
\begin{scriptsize}
\begin{sc}
\setlength\tabcolsep{2pt}
\begin{tabular}{|l|r|c|}
    \hline
    \textbf{CNN Architecture} & $M=\max_{k}n_k$ & $M>N$ \\
    \hline
    VGG(A-E) \citep{VGG} & $3000\textrm{K} (k=1)$ & yes \\
    \hline
    InceptionV3 \citep{InceptionV3} & $1300\textrm{K} (k=3)$ & yes \\
    \hline
    InceptionV4 \citep{InceptionV4} & $1300\textrm{K} (k=3)$ & yes \\
    \hline
    SqueezeNet \citep{SqueezeNet} & $1180\textrm{K} (k=1)$ & no \\
    \hline
    Enet \citep{Enet} & $1000\textrm{K} (k=1)$ & no \\
    \hline
    GoogLeNet \citep{GoogLeNet} & $800\textrm{K} (k=1)$ & no \\
    \hline
    ResNet \citep{ResNet} & $800\textrm{K} (k=1)$ & no \\
    \hline
    Xception \citep{Xception} & $700\textrm{K} (k=1)$ & no \\
    \hline
\end{tabular}
\end{sc}
\end{scriptsize}
\end{center}
\vskip -0.1in
\end{table}

In order to explain the apparent gap between hardness results and practical performance, 
many interesting theoretical results have been recently developed 
\citep{Andoni2014,Sedghi2015,Janzamin15,GauNgoHei2016,HaeVid2015,Soudry17,Goel2017,Li2017,Zhong2017,Tian2017,Brutzkus2017,Soltanolkotabi2017,Du2017}
in order to identify conditions under which one can guarantee that
local search algorithms like gradient descent converge to the globally optimal solution.
However, it turns out that these approaches are either not practical
as they require e.g. knowledge about the data generating measure, 
or a modification of network structure and objective,
or they are for quite restricted network structures, 
mostly one hidden layer networks, and thus are not able to explain the success of deep networks in general.
For deep linear networks one has achieved a quite complete picture of the loss surface 
as it has been shown that  every local minimum is a global minimum \citep{Baldi88,Kawaguchi16,Freeman2017,Hardt2017,Yun2017}. 
% While this is a highly non-trivial result as the optimization problem is non-convex,
% deep linear networks are not interesting in practice as one efficiently just learns a linear function. 
By randomizing the nonlinear part of a feedforward network
with ReLU activation function and making some additional simplifying assumptions, 
\cite{Choro15} can relate the loss surface of neural networks to a certain spin glass model. 
In this model the objective of local minima is close to the global optimum 
and the number of bad local minima decreases quickly with the distance to the global optimum. 
This is a very interesting result but is based on a number of unrealistic assumptions \citep{ChoroJLMR15}. 
More recently, \cite{Quynh2017} have analyzed deep fully connected networks with general activation functions 
and could show that almost every critical point is a global minimum 
if one layer has more neurons than the number of training points. 
While this result holds for networks in practice, it requires a quite extensively overparameterized network.

%On one hand, this is a significant imporvement over previous work 
%as it can be directly applied to the actual problems in deep learning without any simplifying assumptions 
%on the network or on the training data.
%% Moreover, the result also shows that width plays an important role in determining the loss landscape of neural networks
%% in the sense that the loss surface has a rather simple structure with potentially no bad local minima.
%On the other hand, the results are derived only for fully connected networks
%where the number of variables is potentially very large as one of the hidden layers gets larger.

In this paper we overcome the restriction of previous work in several ways.  
This paper is one of the first ones, which analyzes the loss surface of deep CNNs. 
CNNs are of high practical interest as they learn very useful representations \citep{Zeiler2014,Mahendran2015,Yosinki2015} 
with small number of parameters. 
We are only aware of \cite{CohenICML2016} who study the expressiveness of CNNs 
with max-pooling layer and ReLU activation but with rather unrealistic filters (just $1\times 1$) and no shared weights. 
In our setting we allow as well max pooling and general activation functions. 
Moreover, we can have an arbitrary number of filters and we study general convolutions 
as the filters need not be applied to regular
structures like $3\times 3$ but can be patch-based where the only condition is that 
all the patches have the size of the filter. Convolutional layers, 
fully connected layers and max-pooling layers can be combined in almost arbitrary order.  
We study in this paper the expressiveness and loss surface of a CNN where one layer is wide, 
in the sense that it has more neurons than the number of training points. 
While this assumption sounds at first quite strong, 
we want to emphasize that the popular VGG  \citep{VGG} and Inception networks  \citep{InceptionV3, InceptionV4}, see Table \ref{tab:net_width}, fulfill this condition.  
We show that wide CNNs produce linearly independent feature representations at the wide layer and thus are able to fit
the training data exactly (universal finite sample expressivity). 
This is even true with probability one when all the parameters up to the wide layer are chosen randomly\footnote{for any probability measure on the parameter space which
has a density with respect to the Lebesgue measure}.
We think that this explains partially the results of \cite{Zhang2017} where they show experimentally
for several CNNs that they are able to fit random labels. 
Moreover, we provide necessary and sufficient conditions for global minima
with zero squared loss and show for a particular class of CNNs that almost all critical points are globally optimal, 
which to some extent explains why wide CNNs can be optimized so efficiently.
All proofs are moved to the appendix due to limited space.

\section{Deep Convolutional Neural Networks}\label{sec:cnn}
We first introduce our notation and definition of CNNs.
Let $N$ be the number of training samples and 
denote by $X=[x_1,\ldots,x_N]^T\in\RR^{N\times d}, Y=[y_1,\ldots,y_N]^T\in\RR^{N\times m}$
the input resp. output matrix for the training data $(x_i,y_i)_{i=1}^N$, where $d$ is the input dimension
and $m$ the number of classes.

Let $L$ be the number of layers of the network, 
where each layer is either a convolutional, max-pooling or fully connected layer.
The layers are indexed from $k=0,1,\ldots,L$
which corresponds to input layer, 1st hidden layer, $\ldots$, and output layer.
Let $n_k$ be the width of layer $k$ and 
$f_k:\RR^d\to\RR^{n_k}$ the function that computes for every input its feature vector at layer $k$.

The convolutional layer consists of a set of patches of equal size
where every patch is a subset of neurons from the same layer.
Throughout this paper, we assume that the patches of every layer cover the whole layer, 
\ie every neuron belongs to at least one of the patches, 
and that there are no patches that contain exactly the same subset of neurons.
This means that if one patch covers the whole layer then it must be the only patch of the layer.
Let $P_k$ and $l_k$ be the number of patches resp. the size of each patch at layer $k$ for every $0\leq k< L.$
For every input $x\in\RR^d$,
let $\Set{x^1,\ldots,x^{P_0}}\in\RR^{l_0}$ denote the set of patches at the input layer
and $\Set{f_k^1(x),\ldots,f_k^{P_k}(x)}\in\RR^{l_k}$ the set of patches at layer $k$. Each filter of the layer 
consists of the same set of patches.
We denote by $T_k$ the number of convolutional filters 
and by $W_k=[w_k^1,\ldots,w_k^{T_k}]\in\RR^{l_{k-1}\times T_k}$ 
the corresponding parameter matrix of the convolutional layer $k$ for every $1\leq k< L$. 
Each column of $W_k$ corresponds to one filter. 
Furthermore, $b_k\in\RR^{n_k}$ denotes the bias vector and $\sigma_k:\RR\to\RR$ the activation function for each layer.
Note that one can use the same activation function for all layers
but we use the general form to highlight the role of different layers. In this paper, all functions are applied componentwise,
and we denote by $[a]$ the set of integers $\Set{1,2,\ldots,a}$ and by $[a,b]$ the set of integers from $a$ to $b$.

% We are ready to define convolutional layers.
\begin{definition}[Convolutional layer]\label{def:conv}
    A layer $k$ is called a convolutional layer if its output $f_k(x)\in\RR^{n_k}$ is defined for every $x\in\RR^d$ as
    \begin{align}\label{eq:def_conv}
	f_k(x)_h = \sigma_k \Big( \inner{w_k^t,f_{k-1}^p(x)} + (b_k)_h \Big)
    \end{align}
    for every $p\in[P_{k-1}],t\in[T_k],h%=(p,t)
\bydef(p-1)T_k + t.$
\end{definition}
The value of each neuron at layer $k$ is computed by first taking the inner product between 
a filter of layer $k$ and a patch at layer $k-1$, adding the bias and then applying the activation function.
The number of neurons at layer $k$ is thus $n_k=T_k P_{k-1}$,
which we denote as the width of layer $k$. 
Our definition of a convolutional layer is quite general 
as every patch can be an arbitrary subset of neurons of the same layer 
and thus covers most of existing variants in practice.
%and thus the patches of the same layer could be overlapping or non-overlapping.
%Moreover, the number of patches $P_k$ and the size of each patch $l_k$ can be arbitrary for each layer.

Definition \ref{def:conv} includes the fully connected layer as a special case by using
$P_{k-1}=1,l_{k-1}=n_{k-1},f_{k-1}^1(x)=f_{k-1}(x)\in\RR^{n_{k-1}}, 
T_{k}=n_{k}, W_k\in\RR^{n_{k-1}\times n_k}, b_k\in\RR^{n_k}.$
%Basically, we have configured layer $k-1$ so that it has 
Thus we have only one patch which is the whole feature vector at this layer.
%that is, $f_{k-1}^1(x)=f_{k-1}(x)$ and there are $n_k$ filters at layer $k$ .%where each filter corresponds to an incoming weight vector of a neuron at layer $k$.
\begin{definition}[Fully connected layer]\label{def:fc}
    A layer $k$ is called a fully connected layer if its output $f_k(x)\in\RR^{n_k}$ is defined for every $x\in\RR^d$ as
    \begin{align}\label{eq:def_fc}
	f_k(x) = \sigma_k \Big( W_k^T f_{k-1}(x) + b_k \Big) .
    \end{align}
\end{definition}
For some results we also allow max-pooling layers.%  apart from convolutional and fully connected layers.
\begin{definition}[Max-pooling layer]\label{def:mp}
    A layer $k$ is called a max-pooling layer if its output $f_k(x)\in\RR^{n_k}$ is defined for every $x\in\RR^d$ and $p\in[P_{k-1}]$ as
    \begin{align}\label{eq:def_mp}
	f_k(x)_p = \max\Big( {(f_{k-1}^p(x))}_{1},\ldots,{(f_{k-1}^p(x))}_{l_{k-1}} \Big).
    \end{align}
   % for every $p\in[P_{k-1}].$
\end{definition}
A max-pooling layer just computes the maximum element of every patch from the previous layer.
%while skipping everything else.
Since there are $P_{k-1}$ patches at layer $k-1$, the number of output neurons at layer $k$ is $n_k=P_{k-1}.$

\paragraph{Reformulation of Convolutional Layers:}
For each convolutional or fully connected layer,
we denote by $\mathcal{M}_k:\RR^{l_{k-1}\times T_k}\to\RR^{n_{k-1}\times n_k}$ the linear map 
that returns for every parameter matrix $W_k\in\RR^{l_{k-1}\times T_k}$ 
the corresponding full weight matrix $U_k=\mathcal{M}_k(W_k)\in\RR^{n_{k-1}\times n_k}.$
For convolutional layers, $U_k$ can be seen as the counterpart of the weight matrix $W_k$ in fully connected layers.
We define $U_k=\mathcal{M}_k(W_k)=W_k$ if layer $k$ is fully connected.
%Otherwise, if layer $k$ is a convolutional layer then 
%one can see $W_k$ as the set of convolutional filters of layer $k$ while $U_k$ is just a linear function of $W_k.$
Note that the mapping $\mathcal{M}_k$ depends on the patch structure of each convolutional layer $k$.
For example, suppose that layer $k$ has two filters of length $3$, that is, 
$W_k=[w_k^1, w_k^2]=\begin{bmatrix}
	a&d\\
	b&e\\
	c&f
\end{bmatrix}$, and $n_{k-1}=5$ and patches given by a 1D-convolution with stride 1 and no padding
then: 
    $$U_k^T=\mathcal{M}_k(W_k)^T=\begin{bmatrix}
	a&b&c&0&0\\
	d&e&f&0&0\\
	0&a&b&c&0\\
	0&d&e&f&0\\
	0&0&a&b&c\\
	0&0&d&e&f
    \end{bmatrix}.$$
The above ordering of the rows of $U_k^T$ of a convolutional layer is determined by \eqref{eq:def_conv}, 
in particular, the row index $h$ of $U_k^T$ is calculated as $h=%(p,t)\bydef
(p-1)T_k + t$,
which means for every given patch $p$ one has to loop over all the filters $t$ and
compute the corresponding value of the output unit by taking the inner product of the $h$-th row of $U_k^T$ 
with the whole feature vector of the previous layer. 
We assume throughout this paper that that there is no non-linearity at the output layer.
By ignoring max-pooling layers for the moment, the feature maps $f_k:\RR^d\to\RR^{n_k}$ can be written as
\begin{align*}
    f_k(x) = 
    \begin{cases}
	x & k=0\\
	\sigma_k\big(g_k(x)\big) & 1\leq k\leq L-1\\
	g_L(x) & k=L
    \end{cases}
\end{align*}
where $g_k:\RR^d\to\RR^{n_k}$  is the pre-activation function:
\begin{align*}
    g_k(x) = U_k^T f_{k-1}(x) + b_k, \quad \forall 1\leq k\leq L
\end{align*}
By stacking the feature vectors of layer $k$ of all training samples, before and after applying the activation function, 
into a matrix, we define:
\begin{align*}
    F_k=[f_k(x_1),\ldots,f_k(x_N)]^T\in\RR^{N\times n_k},\\
    G_k=[g_k(x_1),\ldots,g_k(x_N)]^T\in\RR^{N\times n_k} .
\end{align*}
In this paper, we refer to $F_k$ as the output matrix at layer $k$. 
It follows from above that
\begin{align}\label{eq:F_k}
    F_k = 
    \begin{cases}
	X & k=0\\
	\sigma_k(G_k) & 1\leq k\leq L-1\\
	G_L & k=L
    \end{cases}
\end{align}
where $G_k=F_{k-1} U_k + \ones_N b_k^T$ for every $1\leq k\leq L.$

In this paper, we assume the following general condition on the structure of convolutional layers.
\begin{assumptions}[Convolutional Structure]\label{ass:conv_structure}
    For every convolutional layer $k$, 
    there exists at least one parameter matrix $W_k\in\RR^{l_{k-1}\times T_k}$ for which 
    the corresponding weight matrix $U_k=\mathcal{M}_k(W_k)\in\RR^{n_{k-1}\times n_k}$ has full rank.
\end{assumptions}
It is straightforward to see that Assumption \ref{ass:conv_structure} is satisfied 
if every neuron belongs to at least one patch and there are no identical patches. 
% This leads to the following result.
As the set of full rank matrices is a dense subset, the following result follows immediately.
%One can see that Assumption \ref{ass:conv_structure} is almost always satisfied if
%every neuron of a convolutional layer $k$ belongs to at least one of the patches.
%It is because otherwise $U_k$ would always have at least one zero column for every possible parameter matrix $W_k$ 
%and thus always has low rank.
%For fully connected layers it holds that $U_k=W_k$, and thus Assumption \ref{ass:conv_structure} is always satisfied
%as the set of full rank matrices is dense in the space of all matrices.
%Assumption \ref{ass:conv_structure} leads to the fact that the set of parameter matrices $W_k$
%for which $U_k$ has full rank is dense in $\RR^{l_{k-1}\times T_k}.$
\begin{lemma}\label{lem:WU_measure_zero}
    If Assumption \ref{ass:conv_structure} holds, then
    for every convolutional layer $k$, 
    the set of $W_k\in\RR^{l_{k-1}\times T_k}$ for which %the corresponding %weight matrix 
    $U_k=\mathcal{M}_k(W_k)\in\RR^{n_{k-1}\times n_k}$ does not have full rank has Lebesgue measure zero.
\end{lemma}

%%%%%%%%%%%%%%%%%%%%%%%%%%%%%%%%%%%%%%%%%%%%%%%%%%%%%%%%%%%
\section{CNN Learn Linearly Independent Features}\label{sec:linear_separability}
%\section{Which CNNs Can Learn Linearly Independent Features?}\label{sec:linear_separability}
% \begin{table}[t]
% \caption{The width of the first convolutional layer of several CNN architectures ($n_1$).
% In all cases, training samples are projected to the feature space of the first hidden layer 
% where the dimension is almost of the order of the size of ImageNet: $N\approx 1200K.$
% }
% \begin{center}
% 	\begin{tabular}{|l|r|r|r|}
% 	    \hline
% 	    \textbf{CNN Architecture} & Input dim & $n_1$ \\
% % 	    \hline
% % 	    AlexNet (Krizhevsky et al, 2012) & 150K & 290K \\
% 	    \hline
% 	    VGG(A-E)  & 150K & \textbf{3000K} \\
% 	    \hline
% 	    GoogLeNet & 150K & 800K \\
% 	    \hline
% 	    InceptionV3 &  260K & 700K \\
% 	    \hline
% 	    ResNet & 150K & 800K  \\
% 	    \hline
% % 	    InceptionV4 (Szegedy et al, 2016) & 260K & 470K \\
% % 	    \hline
% 	    SqueezeNet & 150K & 1100K \\
% 	    \hline
% 	    Enet & 260K & 1000K \\
% 	    \hline
% 	    Xception & 260K & 700K \\
% 	    \hline
% 	\end{tabular}
% \end{center}
% \label{tab:net_width}
% \end{table}
In this section, we show that a class of standard CNN architectures 
with convolutional layers, fully connected layers and max-pooling layers 
plus standard activation functions like ReLU, sigmoid, softplus, etc
are able to learn linearly independent features at every wide hidden layer if it has more neurons than the number of training samples.
Our assumption on training data is the following.
\begin{assumptions}[Training data]\label{ass:different_patches}
    The patches of different training samples are non-identical, that is,
    $x_i^p\neq x_j^q$ for every $p,q\in[P_0],i,j\in[N],i\neq j.$
\end{assumptions}
Assumption \ref{ass:different_patches} is quite weak, 
especially if the size of the input patches is large.
If the assumption does not hold, one can add a small perturbation to the training samples: $\Set{x_1+\epsilon_1,\ldots,x_N+\epsilon_N}.$
The set of $\Set{\epsilon_i}_{i=1}^{N}$ where Assumption \ref{ass:different_patches} 
is not fulfilled for the new dataset has measure zero.
Moreover, $\Set{\epsilon_i}_{i=1}^{N}$ can be chosen arbitrarily small 
so that the influence of the perturbation is negligible. %their influence on the training data is negligible.
Our main assumptions on the activation function of the hidden layers are the following.
\begin{assumptions}[Activation function]\label{ass:activation}
    The activation function $\sigma$ is continuous, non-constant, and satisfies one of the following conditions:
    \begin{itemize}
	\item There exist $\mu_{+},\mu_{-}\in\RR$ s.t. 
	$\lim\limits_{t\to-\infty} \sigma_k(t)=\mu_{-}$ and $\lim\limits_{t\to\infty} \sigma_k(t)=\mu_{+}$  
	and $\mu_{+} \mu_{-}=0$ 
	\item There exist $\rho_1,\rho_2,\rho_3,\rho_4 \in \RR_+$ 
	s.t.  $|\sigma(t)|\leq \rho_1 e^{\rho_2 t}$ for $t< 0$ 
	and $|\sigma(t)|\leq \rho_3 t + \rho_4 $ for $t\geq 0$
       %$|\sigma(t)|\leq \begin{cases}  \rho_1 e^{\rho_2 t}\;\textrm{ for }\;t< 0\\ \rho_3 t + \rho_4  \;\textrm{ for }\;t\geq 0\end{cases}$
    \end{itemize}
\end{assumptions}
Assumption \ref{ass:activation} covers several standard activation functions.
\begin{lemma}\label{lem:activation_exp}
    The following activation functions satisfy Assumption \ref{ass:activation}:
%     \[ \textrm{ReLU: } \sigma(t)=\max(0,t), \textrm{ Sigmoid: }\sigma(t)=\frac{1}{1+e^{-t}}, \textrm{ Softplus: }\sigma_\alpha(t)=\frac{1}{\alpha}\ln(1+e^{\alpha t}) \textrm{ for }\alpha>0.\]
%     \begin{multicols}{2}
   \begin{itemize}
	\item ReLU: $\sigma(t)=\max(0,t)$ 
	\item Sigmoid: $\sigma(t)=\frac{1}{1+e^{-t}}$
% 	\item Tanh:  $\sigma(t)=\frac{2}{1+e^{-2t}}-1$
	\item Softplus: $\sigma_\alpha(t)=\frac{1}{\alpha} \ln(1+e^{\alpha t})$ for some $\alpha>0$
% 	for $\alpha>0$
% 	\item Gaussian: $\sigma(t)=e^{-\frac{(t-\mu)^2}{\delta^2}}$
   \end{itemize}
%     \end{multicols}
\end{lemma}
It is known that the softplus function is a smooth approximation of ReLU.
In particular, it holds that:
\begin{align}\label{eq:relu_softplus}
    \lim\limits_{\alpha\to\infty}\sigma_\alpha(t)
    =\lim\limits_{\alpha\to\infty} \frac{1}{\alpha} \ln(1+e^{\alpha t})
    =\max(0,t) .
\end{align}
%In this paper, Assumption \ref{ass:activation} is only required
%for the bottom layer(s) of the network.
%For all the other layers, the activation functions do not necessarily need to satisfy Assumption \ref{ass:activation},
%but we decided to present the proofs in this paper under the this assumptions for better readability.
% For instance, one can use other continuous monotonic activation functions like Leaky-ReLU
% for all the layers above certain layer $k$ which is discussed in the rest of the paper.
The first main result of this paper is the following. 
\begin{theorem}[Linearly Independent Features]\label{theo:linear_independence}
    Let Assumption \ref{ass:different_patches} hold for the training sample. 
    Consider a deep CNN architecture for which there exists some layer $1\leq k\leq L-1$ such that
    \begin{enumerate}
	\item Layer $1$ and layer $k$ are convolutional or fully connected
	while all the other layers can be convolutional, fully connected or max-pooling
	\item The width of layer $k$ is larger than the number of training samples, $n_k=T_k P_{k-1} \geq N$
% 	\item Every neuron at layer $k$ has a separate bias, that is, $b_k\in[n_k]$ while neurons at other layers could have shared or separate biases
	\item $(\sigma_1,\ldots,\sigma_k)$ satisfy Assumption \ref{ass:activation} 
    \end{enumerate}
    Then there exists a set of parameters of the first $k$ layers $(W_l,b_l)_{l=1}^k$
    such that the set of feature vectors $\Set{f_k(x_1),\ldots,f_k(x_N)}$ are linearly independent.
    Moreover, $(W_l,b_l)_{l=1}^k$ can be chosen in such a way that all the weight matrices $U_l=\mathcal{M}_l(W_l)$
    have full rank for every $1\leq l\leq k.$
\end{theorem}

Theorem \ref{theo:linear_independence} implies that a large class of CNNs employed in practice
with standard activation functions like ReLU, sigmoid or softplus
can produce linearly independent features at any hidden layer 
if its width is larger than the size of training set.
Figure \ref{fig:cnn} shows an example of a CNN architecture that 
satisfies the conditions of Theorem \ref{theo:linear_independence} at the first convolutional layer.
Note that if a set of vectors is linearly independent then they are also linearly separable.
In this sense, Theorem \ref{theo:linear_independence} suggests that CNNs 
can produce linearly separable features at every wide hidden layer.

Linear separability in neural networks has been recently studied by \cite{An2015},
where the authors show that a two-hidden-layer fully connected network with ReLU activation function
can transform any training set to be linearly separable
while approximately preserving the distances of the training data at the output layer. %of any two vectors at the output layer with a factor ranging from $0.5$ to $1$.
Compared to \cite{An2015} %which are derived for fully connected network with ReLU activation function,
our Theorem \ref{theo:linear_independence} is derived for CNNs with a wider range of activation functions.
Moreover, our result shows even linear independence of features 
which is stronger than linear separability.
Recently, \cite{Quynh2017} have shown a similar result for fully connected networks and analytic activation functions.
%Compaprevious results on fully connected networks of similar width, 
%Note that CNNs have the big advantage that their number of training parameters is %often 
%significantly smaller
%as they can take advantage of shared weights and sparse connection.

We want to stress that, in contrast to fully connected networks, for CNNs the condition $n_k\geq N$ of Theorem \ref{theo:linear_independence}  
does not imply that the network has a huge number of parameters
as the layers $k$ and $k+1$ can be chosen to be convolutional. 
In particular, the condition $n_k=T_k P_{k-1}\geq N$ can be fulfilled by increasing the number of filters $T_k$ 
or by using a large number of patches $P_{k-1}$ (however $P_{k-1}$ is upper bounded by $n_k$),
which is however only possible if $l_{k-1}$ is small as otherwise our condition on the patches cannot be fulfilled.
In total the CNN has only $l_{k-1}T_k+l_k T_{k+1}$ parameters versus $n_k(n_{k-1}+n_{k+1})$ for the fully
connected network from and to layer $k$.
%Moreover, one can always reduce the number of training parameters of layer $k$
%by using small patches at layer $k-1$ (\ie $l_{k-1}$ is small) 
%together with small strides (\ie the distance between two consecutive patches).
%In this case, the number of parameters $T_k l_{k-1}$ becomes small as $l_{k-1}$ is small.
%However, the number of patches $P_{k-1}$ can be chosen to increase such that  $n_{k-1}\leq P_{k-1} l_{k-1}$\textbf{what is the purpose of the statement about $n_{k-1}$ ?}
%and thus the condition $n_k=T_k P_{k-1}\geq N$ is more likely to be satisfied.
Interestingly, the VGG-Net \citep{VGG}, where in the first layer small $3\times 3$ filters and stride $1$  is used,
fulfills for ImageNet the condition $n_k\geq N$ for $k=1$, as well as the  Inception networks  \citep{InceptionV3, InceptionV4}, see Table \ref{tab:net_width}.
%have lead to significant improvements in several computer vision tasks and  VGG-Net that won the 1st and 2nd places 
%in the localization and classification tracks of ImageNet Challenge 2014.
%Overall, %Assumption \ref{ass:architecture}
%Theorem \ref{theo:linear_independence}  can be seen as a theoretical support for the usage of small filters and strides in 
%practical CNN architectures %as they can provide us with at least two advantages, 
%%that are, to 
%as it increases the chance of achieving linear separability at the first hidden layers and also
%reduces the total number of training parameters.
%%and to increase the chance of achieving linear separability at early hidden layers in the network.
%The reason why linear separability helps will be discussed in Section \ref{sec:loss_surface} when we analyze the
%loss surface of the CNNs. Note also that the condition $n_k\geq N$ is a sufficient condition but not necessary to prove our results.
%In particular, we conjecture that linear separability % the result of Theorem \ref{theo:linear_independence} 
%might hold with far less number of neurons in practical applications.
%% Still, Table \ref{tab:net_width} reveals an interesting fact about practical networks,
%% that is, many of existing CNN architectures already satisfy 
%% or approximately satisfy this condition from the very first convolutional layer.
%% In particular, their first hidden layers are relatively wide - they are almost of the order of the size of ImageNet.

One might ask now how difficult it is %for a practical learning algorithm 
%like gradient descent 
to find such parameters which generate linearly independent features at a hidden layer?
Our next result shows that once analytic\footnote{A function $\sigma:\RR\to\RR$ is real analytic
if its Taylor series about $x_0$ converges to $\sigma(x_0)$ on some neighborhood of $x_0$
for every $x_0\in\RR$ \citep{KraPar2002}.} activation functions, \eg sigmoid or softplus,
are used at the first $k$ hidden layers of the network,
the linear independence of features at layer $k$ holds with probability $1$ even if one draws 
the parameters of the first $k$ layers $(W_l,b_l)^k$ randomly 
for any probability measure on the parameter space which has a density with respect to the Lebesgue measure.
\begin{theorem}\label{theo:how_often}
    Let Assumption \ref{ass:different_patches} hold for the training samples. 
    Consider a deep CNN for which there exists some layer $1\leq k\leq L-1$ such that
    \begin{enumerate}
	\item Every layer from $1$ to $k$ is convolutional or fully connected
	\item The width of layer $k$ is larger than number of training samples, that is, $n_k=T_k P_{k-1} \geq N$
% 	\item every neuron at layer $k$ has a separate bias, that is, $b_k\in[n_k]$ 
	\item $(\sigma_1,\ldots,\sigma_k)$ are real analytic functions and satisfy Assumption \ref{ass:activation}.
    \end{enumerate}
    Then the set of parameters of the first $k$ layers $(W_l,b_l)_{l=1}^k$ 
    for which the set of feature vectors $\Set{f_k(x_1),\ldots,f_k(x_N)}$ 
    are \textbf{not} linearly independent has Lebesgue measure \textbf{zero}.
\end{theorem}
%Theorem \ref{theo:how_often} can be seen as an extension of Lemma 4.4 in \cite{Quynh2017} 
%from fully connected networks to a large class of CNNs with standard architecture.
Theorem \ref{theo:how_often} is a much stronger statement than Theorem \ref{theo:linear_independence},
as it shows that for almost all weight configurations one gets linearly independent features at the wide layer.
While Theorem \ref{theo:how_often} does not hold for the ReLU activation function as it is not an analytic function, 
we note again that one can approximate the ReLU function arbitrarily well using the softplus function (see \ref{eq:relu_softplus}),
which is analytic function for any $\alpha>0$ and thus Theorem \ref{theo:how_often} applies.
It is an open question if the result holds also for the ReLU activation function itself.
The condition $n_k\geq N$ is not very restrictive as several state-of-the art CNNs , see Table \ref{tab:net_width}, fulfill the condition.
Furthermore, we would like to stress that Theorem \ref{theo:how_often} is \emph{not} true for deep linear networks. The reason is
simply that the rank of a product of matrices can at most be the minimal rank among all the matrices. The nonlinearity of the activation
function is thus critical (note that the identity activation function, $\sigma(x)=x$, does not fulfill Assumption \ref{ass:activation}).
% In the appendix we show an extension of Theorem \ref{theo:how_often} to a variant of 
% densely connected CNN architectures \cite{Huang2017} (in our setting the first layer has to be fully connected). There the condition
% $n_k\geq N$ can be further relaxed to $n_1+\ldots+n_k\geq N$. Thus there is a trade-off between the width and depth of the network.

\begin{table*}[t]
\caption{The smallest singular value $\svmin{F_1}$ of the feature matrix $F_1$ of the first convolutional layer 
(similar $\svmin{F_3}$ for the feature matrix $F_3$ of the second convolutional layer) 
of the trained network in Figure \ref{fig:cnn} are shown for varying number of convolutional filters $T_1$.
The rank of a matrix $A \in \RR^{m\times n}$ is estimated (see Chapter 2.6.1 in \cite{NumRec}) 
by computing the singular value decomposition of $A$ and counting the singular values which 
exceed the threshold $\frac{1}{2}\sqrt{m+n+1}\,\svmax{A}\epsilon$, where $\epsilon$ is machine precision. 
For all filter sizes the feature matrices $F_1$ have full rank.
Zero training error is attained for $T_1\geq 30.$ }
\label{tab:rank}
\vskip 0.15in
\begin{center}
\begin{scriptsize}
\begin{sc}
\begin{tabular}{|c|c|c|c|c|c|c|c|c|c|}
    \hline
    $T_1$ & $\textrm{size}(F_1)$ & $\rank(F_1)$ & $\svmin{F_1}$ & $\textrm{size}(F_3)$ & $\rank(F_3)$ & $\svmin{F_3}$ & Loss\newline($\times 10^{-5}$) & Train error & Test error \\
    \hline
    10 & $60000\times 6760$ & \textbf{6760} & \num{3.7e-6} & $60000\times 2880$ & \textbf{2880} & \num{2.0e-2} & \num{2.4} & 8 & 151 \\
    \hline
    20 & $60000\times 13520$ & \textbf{13520} & \num{2.2e-6} & $60000\times 2880$ & \textbf{2880} & \num{7.0e-4} & \num{1.2} & 1 & 132 \\
    \hline
    30 & $60000\times 20280$ & \textbf{20280} & \num{1.5e-6} & $60000\times 2880$ & \textbf{2880} & \num{2.4e-4} & \num{0.24} & \textbf{0} & 174 \\
    \hline
    40 & $60000\times 27040$ & \textbf{27040} & \num{2.0e-6} & $60000\times 2880$ & \textbf{2880} & \num{2.2e-3} & \num{0.62} & \textbf{0} & 124 \\
    \hline
    50 & $60000\times 33800$ & \textbf{33800} & \num{1.3e-6} & $60000\times 2880$ & \textbf{2880} & \num{3.9e-5} & \num{0.02} & \textbf{0} & 143\\
    \hline
    60 & $60000\times 40560$ & \textbf{40560} & \num{1.1e-6} & $60000\times 2880$ & \textbf{2880} & \num{4.0e-5} & \num{0.57} & \textbf{0} & 141 \\
    \hline
    70 & $60000\times 47320$ & \textbf{47320} & \num{7.5e-7} & $60000\times 2880$ & \textbf{2880} & \num{7.1e-3} & \num{0.12} & \textbf{0} & 120 \\
   \hline
    80 & $60000\times 54080$ & \textbf{54080} & \num{5.4e-7} & $60000\times 2880$ & 2875 & \num{4.9e-18} & \num{0.11} & \textbf{0} & 140 \\
    \hline
    89 & $60000\times 60164$ & \textbf{60000} & \num{2.0e-8} & $60000\times 2880$ & \textbf{2880} & \num{8.9e-10} & \num{0.35} & \textbf{0} & 117 \\ 
    \hline
    100 & $60000\times 67600$ & \textbf{60000} & \num{1.1e-6} & $60000\times 2880$ & $2856$ & \num{8.5e-27} & \num{0.04} & \textbf{0} & 139 \\
    \hline
\end{tabular}
\end{sc}
\end{scriptsize}
\end{center}
\vskip -0.1in
\end{table*}
% \begin{table*}[t]
% \caption{The loss and the number of misclassified training and test samples (train/test errors) 
% of the corresponding trained networks in Table \ref{tab:rank}. Zero training error is attained for $T_1\geq 30.$}
% \label{tab:train_test_error}
% \vskip 0.15in
% \begin{center}
% \begin{scriptsize}
% \begin{sc}
% \begin{tabular}{|l|c|c|c|c|c|c|c|c|c|c|c|}
%     \hline
%     $T_1$ & 10 & 20 & 30 & 40 & 50 & 60 & 70 & 80 & 89 & 100 \\
%     \hline 
%     Train loss ($\times 10^{-5}$) & \num{2.4} & \num{1.2} & \num{0.24} & \num{0.62} & \num{0.02} & \num{0.57} & \num{0.12} & \num{0.11} & \num{0.35} & \num{0.04} \\
%     \hline
%     Train error & 8 & 1 & \textbf{0} & \textbf{0} & \textbf{0} & \textbf{0} & \textbf{0} & \textbf{0} & \textbf{0} & \textbf{0} \\
%     \hline
%     Test error & 151 & 132 & 174 & 124 & 143 & 141 & 120 & 140 & 117 & 139 \\
%     \hline
% \end{tabular}
% \end{sc}
% \end{scriptsize}
% \end{center}
% \vskip -0.1in
% \end{table*}

To illustrate Theorem \ref{theo:how_often}
we plot the rank of the feature matrices of the network in Figure \ref{fig:cnn}. 
We use the MNIST dataset with $N=60000$ training and $10000$ test samples.
We add small Gaussian noise $\mathcal{N}(0,10^{-5})$ to every training sample so that Assumption \ref{ass:different_patches} is fulfilled.
We then vary the number of convolutional filters $T_1$ of the first layer from $10$ to $100$ 
and train the corresponding network with squared loss and sigmoid activation function
using Adam \citep{Kingma2015} %with default hyperparameters % ($\textrm{learning rate}=0.001, \beta_1=0.9, \beta_2=0.999,\epsilon=\num{e-8}$).
and decaying learning rate for 2000 epochs.
In Table \ref{tab:rank} we show the smallest singular value of the feature matrices
together with the corresponding training loss, training and test error.
If number of convolutional filters is large enough (\ie $T_1\geq 89$), one has
$n_1=26\times 26\times T_1\geq N=60000$, and the second condition of Theorem  \ref{theo:how_often}  is satisfied for $k=1$.
Table \ref{tab:rank} shows that the feature matrices $F_1$ have full rank in all cases (and $F_3$ in almost all cases),
in particular for $T_1\geq 89$ as shown in Theorem \ref{theo:how_often}.
As expected when the feature maps of the training samples are linearly independent after the first layer ($F_1$ has rank $60000$ for $T\geq 89$)
the training error is zero and the training loss is close to zero (the GPU uses single precision). However, as linear independence is stronger than linear separability one can achieve already for $T<89$  zero training error.
%Remarkably, Table \ref{tab:train_test_error} shows that the network can even obtain zero training error 
%when the condition on the wide layer is not fulfilled (\ie $T_1<89$),
%which somehow indicates that our condition $n_k\geq N$ from Assumption \ref{ass:architecture} is still very loose and that 
%in practice the network might need significantly less number of neurons.
% This implies that the algorithm has frequently converged to critical points where the feature matrix at the wide layer has full rank.

It is interesting to note that Theorem \ref{theo:how_often} explains previous empirical observations.
In particular, \cite{Czarnecki2017} have shown empirically that linear separability is often obtained 
already in the first few hidden layers of the trained networks.
This is done by attaching a linear classifier probe \citep{Alain2016} to every hidden layer in the network
after training the whole network with backpropagation.
%We think that Theorem \ref{theo:linear_independence} and Theorem \ref{theo:how_often} can be seen as an explanation 
%for this phenomenon in the sense that a class of practical CNNs with standard architecture
%can easily learn linearly separable features at their hidden layers.
%This, as shown in Theorem \ref{theo:how_often}, 
%can even occur with probability $1$ (up to a set of measure zero) if all the convolutional filters of the first $k$ layers
%are drawn at random.
The fact that  Theorem \ref{theo:how_often} holds even if the parameters of the bottom layers up to the wide layer $k$ 
are chosen randomly is also in line with recent empirical observations for CNN architectures that one has little loss
in performance if the weights of the initial layers are chosen randomly without training
%Interestingly, recent work have shown that certain CNN architectures can achieve very good performance 
%on certain tasks even with random untrained weights 
\citep{Jarrett09,Saxe11,Yosinki2014}.%{Giryes2016,He2016,Gilbert2017,Neyshabur2017}.

As a corollary of Theorem \ref{theo:linear_independence} we get the following universal finite sample expressivity for CNNs.
In particular, a deep CNN with scalar output can perfectly fit any scalar-valued function for a finite number of inputs
if the width of the last hidden layer is larger than the number of training samples.
\begin{corollary}[Universal Finite Sample Expressivity]\label{cor:finite_expressivity}
    Let Assumption \ref{ass:different_patches} hold for the training samples.
    Consider a standard CNN with scalar output which satisfies the conditions of Theorem \ref{theo:linear_independence} 
    at the last hidden layer $k=L-1.$ 
    Let $f_L:\RR^d\to\RR$ be the output of the network given as
    \begin{align*}
	f_L(x) = \sum_{j=1}^{n_{L-1}} \lambda_j f_{(L-1)j}(x) \quad\forall x\in\RR^d
    \end{align*}
    where $\lambda\in\RR^{n_{L-1}}$ is the weight vector of the last layer.
    Then for every target $y\in\RR^N$, 
    there exists $\Set{\lambda, (W_l,b_l)_{l=1}^{L-1}}$ so that it holds $f_L(x_i)=y_i$ for every $i\in[N].$
\end{corollary}
The expressivity of neural networks has been well-studied, in particular in the universal approximation theorems for one hidden layer networks \citep{Cybenko1989,Hornik1989}. 
Recently, many results have shown why deep networks are superior to shallow networks in terms of
expressiveness
% \citep{Telgarsky2016,Eldan2016,Safran2017,Yarotsky2016,Liang2017,Mhaskar2016,Montufar2014,Raghu2017}. 
\citep{Delalleau2011,Telgarsky2016,Telgarsky2015,Eldan2016,Safran2017,Yarotsky2016,Poggio2016,Liang2017,Mhaskar2016,Montufar2014,
Pascanu2014,Raghu2017}
While most of these results are derived for fully connected networks, 
it seems that \cite{CohenICML2016} are the first ones who study expressivity of CNNs.
In particular, they show that CNNs with max-pooling and ReLU units are universal in the sense that
they can approximate any given function if the size of the networks is unlimited.
However, the number of convolutional filters in this result has to grow exponentially with the number of patches and they
do not allow shared weights in their result, which is a standard feature of CNNs. Corollary \ref{cor:finite_expressivity} shows universal finite sample expressivity, instead of universal function approximation, 
even for $L=2$ and $k=1$, that is
a single convolutional layer network can perfectly fit the training data 
as long as the number of hidden units is larger than the number of training samples. 
% To the best of our knowledge, this is the first result on universal finite sample expressivity 
% for a large class of practical CNNs.

For fully connected networks, universal finite sample expressivity has been studied by \cite{Zhang2017,Quynh2017,Hardt2017}.
It is shown that a single hidden layer fully connected network with $N$ hidden units can express any training set of size $N$.
While the number of training parameters of a single hidden layer CNN with $N$ hidden units and scalar output 
is just $2N+T_1l_0$, where $T_1$ is the number of convolutional filters and $l_0$ is the size of each filter,
it  is $Nd+2N$ for fully connected networks.
If we set the width of the hidden layer of the CNN as $n_1=T_1P_0=N$ in order to fulfill the condition of Corollary \ref{cor:finite_expressivity}, then the number of training parameters of the CNN 
becomes $2N+Nl_0/P_0$, which is less than $3N$ if $l_0\leq P_0$ compared to $(d+2)N$ for the fully connected case. In practice one almost always has $l_0\leq P_0$ as $l_0$ is typically a small integer
and $P_0$ is on the order of the dimension of the input.
Thus, the number of parameters to achieve universal finite sample expressivity is for CNNs significantly smaller than 
for fully connected networks.

Obviously, in practice it is most important that the network generalizes rather than just fitting the training data.
By using shared weights and sparsity structure, CNNs seem to implicitly regularize the model
to achieve good generalization performance. Thus even though they can fit also random labels or noise  \citep{Zhang2017}
due to the universal finite sample expressivity shown in Corollary \ref{cor:finite_expressivity}, they seem still to be able
to generalize well \citep{Zhang2017}.

%In practice, what is more important for a general neural network is not just about the ability to fit the training data,
%but also the ability to generalize over unseen data. 
%By using shared weights and sparsity structure, CNNs might implicitly regularize the model class in certain ways
%in order to achieve good generalization performance.
%Even so, it seems quite counter-intuitive that these nets can still easily fit random training labels or even noise data \citep{Zhang2017}.
%We hope that Corollary \ref{cor:finite_expressivity} might shed some light on this phenomenon
%given the fact that practical CNNs nowadays often have a very wide hidden layer in their architecture.

%%%%%%%%%%%%%%%%%%%%%%%%%%%%%%%%%%%%%%%%%%%%%%%%%%%%%%%%%%%
\section{Optimization Landscape of Deep CNNs}\label{sec:loss_surface}
In this section, we restrict our analysis to the use of least squares loss.
However, as we show later that the network can produce exactly the target output (\ie $F_L=Y$)
for some choice of parameters, all our results can also be extended to any other loss function
where the global minimum is attained at $F_L=Y$, for instance the squared Hinge-loss analyzed in \cite{Quynh2017}.
Let $\mathcal{P}$ denote the space of all parameters of the network.
The final training objective $\Phi:\mathcal{P}\to\RR$ is given as
    \begin{align}\label{eq:training_obj}
	\Phi\Big((W_l,b_l)_{l=1}^L\Big) = \frac{1}{2} \norm{F_L-Y}_F^2 
    \end{align}
    where $F_L$ is defined as in \eqref{eq:F_k}, which is also the same as
    \begin{align*}
	F_L = \sigma_{L-1}(\ldots\sigma_1(XU_1+\ones_N b_1^T)\ldots)U_L + \ones_N b_L^T,
    \end{align*}
    where $U_l=\mathcal{M}_l(W_l)$ for every $1\leq l\leq L.$ We require 
the following assumptions on the architecture of CNN.

\begin{figure*}[ht]
\vskip 0.2in
\begin{center}
    \includegraphics[width=0.7\linewidth]{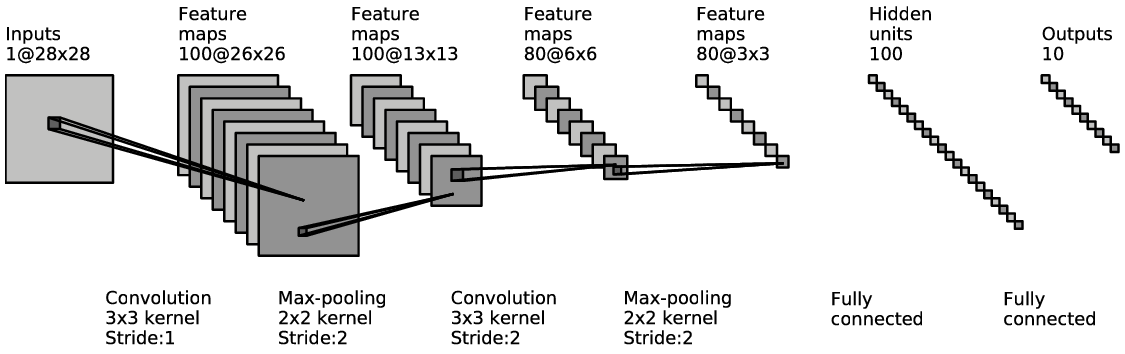}
\end{center}
\caption{
An example of CNN for a given training set of size $N\leq 100\times 26\times 26=67600.$
The width of each layer is $d=n_0=784, n_1=67600, n_2=16900, n_3=2880, n_4=720, n_5=100, n_6=m=10.$
One can see that $n_1\geq N$ and the network has pyramidal structure 
from layer $2$ till the output layer, that is, $n_2\geq\ldots\geq n_6.$}
\label{fig:cnn}
\vskip -0.2in
\end{figure*}

\begin{assumptions}[CNN Architecture]\label{ass:architecture}
    Every layer in the network is a convolutional layer or fully connected layer and the output layer is fully connected.
    Moreover, there exists some hidden layer $1\leq k\leq L-1$ such that the following holds:
    \begin{itemize}
	\item The width of layer $k$ is larger than number of training samples, that is, $n_k=T_k P_{k-1} \geq N$
	\item All the activation functions of the hidden layers $(\sigma_1,\ldots,\sigma_{L-1})$ satisfy Assumption \ref{ass:activation}
	\item $(\sigma_{k+1},\ldots,\sigma_{L-1})$ are strictly increasing or strictly decreasing, and  differentiable
	\item The network is pyramidal from layer $k+1$ till the output layer, that is, $n_{k+1}\geq\ldots\geq n_L$
    \end{itemize}
\end{assumptions}

% Assumption \ref{ass:architecture} is not very restrictive as it covers a lot of recently employed deep CNN architectures.
A typical example that satisfies Assumption \ref{ass:architecture} with $k=1$ can be found in Figure \ref{fig:cnn}
where one disregards max-pooling layers and uses \eg sigmoid or softplus activation function.
% This example is motivated by the fact that the first convolutional layer of existing CNN architectures is pretty wide,
% see \eg Table \ref{tab:net_width}.
% Another example satisfying Assumption \ref{ass:architecture} 
% is the network above except that the same activation function 
% is used across all the layers, \eg sigmoid or softplus.

In the following, let us define for every $1\leq k\leq L-1$ the subset $S_k\subseteq\mathcal{P}$ of the parameter space such that
\begin{align*}
    S_k\bydef \Setbar{(W_l,b_l)_{l=1}^L}{F_k, U_{k+2},\ldots, U_L \textrm{have full rank}}.
\end{align*}
The set $S_k$ is the set of parameters where the feature matrix at layer $k$ and all the weight matrices 
from layer $k+2$ till the output layer have full rank.
In the following, we examine conditions for global optimality in $S_k$.
It is important to note that $S_k$ covers almost the whole parameter space under an additional mild condition
on the activation function.
\begin{lemma}\label{lem:S_k}
    Let Assumption \ref{ass:different_patches} hold for the training samples
    and a deep CNN satisfy Assumption \ref{ass:architecture} for some layer $1\leq k\leq L-1.$
    If the activation functions of the first $k$ layers $(\sigma_1,\ldots,\sigma_k)$ are real analytic,
    then the complementary set $\mathcal{P}\setminus S_k$ has Lebesgue measure zero.
\end{lemma}

In the next key lemma, we bound the objective function in terms of its gradient magnitude w.r.t. the weight matrix of layer $k$
for which $n_k\geq N.$
For every matrix $A\in\RR^{m\times n}$, 
let $\svmin{A}$ and $\svmax{A}$ 
denote the smallest and largest singular value of $A$.
Let $\norm{A}_F=\sqrt{\sum_{i,j}A_{ij}^2}$, 
$\minnorm{A}\bydef\min\limits_{i,j}|A_{ij}|$ and $\maxnorm{A}\bydef\max\limits_{i,j}|A_{ij}|.$
From \eqref{eq:F_k}, and \eqref{eq:training_obj},
it follows that $\Phi$ can be seen as a function of $(U_l,b_l)_{l=1}^L$, and thus we can use $\nabla_{U_{k}}\Phi$. %define in the following
%\begin{align*}
%    \nabla_{U_{k}}\Phi 
%%     = \frac{\partial\Phi }{\partial U_{k}}\Big|_{ (U_l,b_l)_{l=1}^L = (\mathcal{M}_l(W_l),b_l)_{l=1}^L } 
%    = \frac{\partial\Phi\Big( (U_l,b_l)_{l=1}^L \Big) }{\partial U_{k}} .
%\end{align*}
If layer $k$ is fully connected then $U_k=\mathcal{M}_k(W_k)=W_k$ 
and thus $\nabla_{U_k}\Phi=\nabla_{W_k}\Phi.$
Otherwise, if layer $k$ is convolutional then we note that 
$\nabla_{U_k}\Phi$ is ``not'' the true gradient of the training objective
because $U_k$ is not the true optimization parameter but $W_k.$
In this case, the true gradient of $\Phi$ w.r.t. to the true parameter matrix $W_k$ which consists of convolutional filters
can be computed via the chain rule as
\begin{align*}
\frac{\partial \Phi}{\partial (W_k)_{rs}} = \sum_{i,j} \frac{\partial \Phi}{\partial (U_k)_{ij}}\frac{\partial (U_k)_{ij}}{\partial (W_k)_{rs}}
%    \nabla_{W_k}\Phi 
%    = \frac{\partial\Phi\Big( (U_l,b_l)_{l=1}^L \Big) }{\partial U_{k}} 
%    \frac{\partial U_k}{\partial W_k}  
%    = \frac{\partial\Phi\Big( (U_l,b_l)_{l=1}^L \Big) }{\partial U_{k}} 
%    \frac{\partial \mathcal{M}_k(W_k)}{\partial W_k}  .
\end{align*}
Note that even though we write the partial derivatives w.r.t. the matrix elements, $\nabla_{W_k} \Phi$ resp. $\nabla_{U_k}\Phi$ are
the matrices of the same dimension as $W_k$ resp. $U_k$ in the following.
\begin{lemma}\label{lem:bounds}
    Consider a deep CNN satisfying Assumption \ref{ass:architecture} for some hidden layer $1\leq k\leq L-1$.
    Then it holds 
    \begin{align*}
    & \Fnorm{\nabla_{U_{k+1}}\Phi}\\ \geq &
	\svmin{F_k} 
	\Big( \prod_{l=k+1}^{L-1} \svmin{U_{l+1}} \minnorm{\sigma_l'(G_l)} \Big) 
	\Fnorm{F_L-Y}
    \end{align*}
    and 
    \begin{align*}
	&\Fnorm{\nabla_{U_{k+1}}\Phi}\\ \leq  &
	\svmax{F_k} 
	\Big( \prod_{l=k+1}^{L-1} \svmax{U_{l+1}} \maxnorm{\sigma_l'(G_l)} \Big)
	\Fnorm{F_L-Y}.
    \end{align*}
\end{lemma}
Our next main result is motivated by the fact that empirically when training over-parameterized neural networks with shared weights and sparsity structure like CNNs, there seem to be no problems with sub-optimal local minima.
In many cases, even when training labels are completely random, 
local search algorithms like stochastic gradient descent can converge to a solution 
with almost zero training error \citep{Zhang2017}.
To understand better this phenomenon, we first characterize in the following Theorem \ref{theo:loss_surface} 
the set of points in parameter space with zero loss,
and then analyze in Theorem \ref{theo:loss_surface_special_case} 
the loss surface for a special case of the network.
We emphasize that our results hold for standard deep CNNs with 
convolutional layers with shared weights and fully connected layers.
\begin{theorem}[Conditions for Zero Training Error]\label{theo:loss_surface}
    Let Assumption \ref{ass:different_patches} hold for the training sample and suppose that the 
    CNN architecture satisfies Assumption \ref{ass:architecture} for some hidden layer $1\leq k\leq L-1$.
    Let $\Phi:\mathcal{P}\to\RR$ be defined as in \eqref{eq:training_obj}.
    Given any point $(W_l,b_l)_{l=1}^L\in S_k.$
    Then it holds that $\Phi\Big( (W_l,b_l)_{l=1}^L \Big)=0$ if and only if $\nabla_{U_{k+1}}\Phi\Big|_{(W_l,b_l)_{l=1}^L}=0.$
\end{theorem}
\begin{proof}
    If $\Phi\Big( (W_l,b_l)_{l=1}^L \Big)=0$ then it follows from the upper bound of Lemma \ref{lem:bounds} 
    that $\nabla_{U_{k+1}}\Phi=0.$
    For reverse direction, one has $(W_l,b_l)_{l=1}^L\in S_k$
    and thus $\rank(F_k)=N$ and $U_{l}$ has full rank for every $l\in[k+2,L].$
    Thus it holds $\svmin{F_k}>0$ and $\svmin{U_l}>0$ for every $l\in[k+2,L].$
    Moreover, $(\sigma_{k+1},\ldots,\sigma_{L-1})$ have non-zero derivative by Assumption \ref{ass:architecture}
    and thus $\minnorm{\sigma_l'(G_l)} > 0$ for every $l\in[k+1,L-1].$
    This combined with Lemma \ref{lem:bounds} leads to 
    $\Phi\Big (W_l,b_l)_{l=1}^L \Big)=\Fnorm{F_L-Y}=0.$
\end{proof}
Lemma \ref{lem:S_k} shows that the set of points which are not covered by Theorem \ref{theo:loss_surface} 
has measure zero if the first $k$ layers have analytic activation functions. The necessary and sufficient condition of Theorem \ref{theo:loss_surface} is rather intuitive
as it requires the gradient of the training objective to vanish w.r.t. the full weight matrix of layer $k+1$
regardless of the architecture of this layer.
It turns out that if layer $k+1$ is fully connected, 
then this condition is always satisfied at a critical point,
in which case we obtain that every critical point in $S_k$ is a global minimum with exact zero training error.
This is shown in the next Theorem \ref{theo:loss_surface_special_case}, 
where we consider a classification task with $m$ classes,
$Z\in\RR^{m\times m}$ is the full rank class encoding matrix e.g. the identity matrix
and $(X,Y)$ the training sample such that $Y_{i:}=Z_{j:}$ whenever the training sample $x_i$ belongs to class $j$
for every $i\in[N],j\in[m].$ 

\begin{theorem}[Loss Surface of CNNs]\label{theo:loss_surface_special_case}
    Let $(X,Y,Z)$ be a training set for which Assumption \ref{ass:different_patches} holds,
    the CNN architecture satisfies Assumption \ref{ass:architecture} for some hidden layer $1\leq k\leq L-1$,
    and layer $k+1$ is fully connected. 
    Let $\Phi:\mathcal{P}\to\RR$ be defined as in \eqref{eq:training_obj}.
    Then the following holds
    \begin{itemize}
        \item Every critical point $(W_l,b_l)_{l=1}^L \in  S_k$ is a global minimum with $\Phi\Big((W_l,b_l)_{l=1}^L\Big) = 0$
	\item There exist infinitely many global minima $(W_l,b_l)_{l=1}^L\in  S_k$ with $\Phi\Big((W_l,b_l)_{l=1}^L\Big) = 0$
    \end{itemize}
\end{theorem}
%Note that under the conditions of Theorem \ref{theo:loss_surface_special_case}, 
%the number of training parameters of layer $k+1$ 
%must satisfy $n_k n_{k+1}\geq N n_{k+1}$,
%which certainly limits the application of the result to large scale datasets.
%However, from the optimization perspective,
Theorem \ref{theo:loss_surface_special_case} shows that 
the loss surface for this type of CNNs has a rather simple structure
in the sense that every critical point in $S_k$ must be a global minimum with zero training error.  Note that if the activation
functions up to layer $k$ are analytic, the complement of $S_k$ has measure zero (see Lemma \ref{lem:S_k}).
%We would like to emphasize again that $S_k$ can cover already almost the whole parameter space $\mathcal{P}$
%under a mild additional assumption that the activation function is real analytic up to layer $k$ (see Lemma \ref{lem:S_k}).
For those critical points lying outside $S_k$, it must hold that either one of the weight matrices $\Set{U_{k+2},\ldots,U_L}$ 
has low rank or the set of feature vectors at layer $k$ is not linearly independent (\ie $F_k$ has low rank).
Obviously, some of these critical points can also be global minima, 
but we conjecture that they cannot be suboptimal local minima due to the following reasons.
First, it seems unlikely that a critical point with a low rank weight matrix is a suboptimal local minimum 
as this would imply that all possible full rank perturbations of the current solution must have larger/equal objective value.
However, there is no term in the loss function which favors low rank solutions.
Even for linear networks, it has been shown by \cite{Baldi88} that all the critical points with low rank 
weight matrices have to be saddle points and thus cannot be suboptimal local minima.
Second, a similar argument applies to the case where one has a critical point outside $S_k$ 
such that the features are not linearly independent.
In particular, any neighborhood of such a critical point contains points 
which have linearly independent features at layer $k$,
from which it is easy to reach zero loss if one fixes the parameters of the first $k$ layers and optimizes the loss w.r.t. the remaining ones.
This implies that every small neighborhood of the critical point 
should contain points from which there exists a descent path that leads to a global minimum with zero loss, 
which contradicts the fact that the critical point is a suboptimal local minimum.
In summary, if there are critical points lying outside the set $S_k$, 
then it is very ``unlikely'' that these are suboptimal local minima but rather also global minima, saddle points or local maxima. 

It remains an interesting open problem if the result of Theorem \ref{theo:loss_surface_special_case} 
can be transferred to the case where layer $k+1$ is also convolutional.
In any case whether layer $k+1$ is fully connected or not, one might assume that a solution 
with zero training error still exists as it is usually the case for practical over-parameterized networks.
However, note that Theorem \ref{theo:loss_surface} shows that at those points where the loss is zero,
the gradient of $\Phi$ w.r.t. $U_{k+1}$ must be zero as well.

An interesting special case of Theorem \ref{theo:loss_surface_special_case} is when the network is fully connected
% (note that Definition \ref{def:conv} of convolutional layers includes fully connected layers as a special case),
in which case all the results of Theorem \ref{theo:loss_surface_special_case} hold without any modifications.
This can be seen as a formal proof for 
the implicit assumption used in the recent work \citep{Quynh2017} 
that there exists a global minimum with absolute zero training error for the class of  fully connected, deep and wide networks.

\section{Conclusion}
We have analyzed the expressiveness and loss surface of CNNs in realistic and practically relevant settings. As state-of-the-art networks
fulfill exactly or approximately the condition to have a sufficiently wide convolutional layer, we think that our results help to understand why
current CNNs can be trained so effectively. It would be interesting to discuss the loss surface for cross-entropy loss, which currently does not
fit into our analysis as the global minimum does not exist when the data is linearly separable.

% Acknowledgements should only appear in the accepted version.
\section*{Acknowledgements}
The authors would like to thank the reviewers for their helpful comments on the paper
and Maksym Andriushchenko for helping us to set up the experiment on the rank of features.

\bibliography{regul}
\bibliographystyle{icml2018}

% In the unusual situation where you want a paper to appear in the
% references without citing it in the main text, use \nocite
% \nocite{langley00}

\ifpaper
\appendix

\section{Tools}
The following key property of real analytic functions is helpful to prove our Theorem \ref{theo:how_often} and Lemma \ref{lem:WU_measure_zero}.
\begin{lemma}\cite{Dan15,Boris15}\label{lem:zeros_of_analytic}
    If $f:\RR^n\to\RR$ is a real analytic function which is not identically zero
    then the set $\Setbar{x\in\RR^n}{f(x)=0}$ has Lebesgue measure zero.
\end{lemma}

\section{Proof of Lemma \ref{lem:WU_measure_zero}}
% \begin{proof}
    Since $U_k=\mathcal{M}_k(W_k)\in\RR^{n_{k-1}\times n_k}$ and $\mathcal{M}_k$ is a linear map, 
    every entry of $U_k$ is a linear function of the entries of $W_k$.
    Let $m=\min(n_{k-1}, n_k)$, then the set of low rank matrices $U_k$ is characterized by a system of equations 
    where the $\binom{\max(n_{k-1}, n_k)}{m}$ determinants of all $m\times m$ sub-matrices of $U_k$ are zero. 
    As the determinant is a polynomial in the entries of the matrix
    and thus a real analytic function, and the composition of analytic functions is again analytic,
    we get that each determinant  is a real analytic function of $W_k$.
    By Assumption \ref{ass:conv_structure}, there exists at least one $W_k$ such that one of these determinants is non-zero.
   Thus by Lemma \ref{lem:zeros_of_analytic},
    the set of $W_k$ where this determinant is zero has Lebesgue measure zero.
    As all the submatrices need to have low rank in order that $U_k$ has low rank, 
    we get that the set of $W_k$ where $U_k$ has low rank has Lebesgue measure zero. 
% \end{proof}

\section{Proof of Lemma \ref{lem:activation_exp}}
% \begin{proof}
    \begin{itemize}
	\item ReLU:
	It holds for every $t<0$ that $\sigma(t)=\max(0,t)=0 < e^{t}$,
	and for $t\geq 0$ that $\sigma(t)=t < t+1.$
	Thus ReLU satisfies the second condition of Assumption \ref{ass:activation}.
	
	\item Sigmoid: 
	It holds that
	\begin{align*}
	    \lim\limits_{t\to-\infty} \frac{1}{1+e^{-t}} = 0,\quad
	    \lim\limits_{t\to\infty} \frac{1}{1+e^{-t}} = 1 .
	\end{align*}
	Thus $\sigma$ satisfies the first condition of Assumption \ref{ass:activation}.
	
% 	\item Tanh:
% 	\begin{align*}
% 	    &\lim\limits_{t\to-\infty} \left(\frac{2}{1+e^{-2t}}-1\right) = -1,\\
% 	    &\lim\limits_{t\to\infty} \left(\frac{2}{1+e^{-2t}}-1\right) = 1 .
% 	\end{align*}
	
	\item Softplus:
	Since $1+e^{\alpha t} \leq 2e^{\alpha t}$ for every $t\geq 0$, it holds for every $t\geq 0$ that
	\begin{align*} 
	    0 \leq \sigma_\alpha(t) 
	    &= \frac{1}{\alpha}\log(1+e^{\alpha t}) \\
	    &\leq \frac{1}{\alpha}\log(2 e^{\alpha t}) \\
	    &= \frac{\log(2)}{\alpha} + t .
	\end{align*}
	Moreover, since $\log(1+t)\leq t$ for $t>0,$ it holds $\log(1+e^{\alpha t}) \leq e^{\alpha t}$ for every $t\in\RR$. 
	Thus it holds that $0 \leq \sigma_\alpha(t) \leq \frac{ e^{\alpha t}}{\alpha}$ for every $t<0.$
	Therefore $\sigma_\alpha$ satisfies the second condition of Assumption \ref{ass:activation}
	for $\rho_1=1/\alpha,\rho_2=\alpha,\rho_3=1,\rho_4=\log(2)/\alpha.$
	
% 	\item Gaussian: It holds that
% 	\begin{align*}
% 	    \lim\limits_{t\to-\infty} e^{-\frac{(t-\mu)^2}{\delta^2}} = 0,\quad
% 	    \lim\limits_{t\to\infty} e^{-\frac{(t-\mu)^2}{\delta^2}} = 0 .
% 	\end{align*}
% 	Thus $\sigma$ satisfies the 1st condition of Assumption \ref{ass:activation}.
    \end{itemize}
% \end{proof}

\section{Proof of Theorem \ref{theo:linear_independence}}
To prove Theorem \ref{theo:linear_independence},
we first show that Assumption \ref{ass:different_patches} 
can be transported from the input to the output layer.
\begin{lemma}\label{lem:transport_different_patches}
    Let Assumption \ref{ass:different_patches} hold for the training sample.
    Consider a standard deep CNN architecture which satisfies the following
    \begin{enumerate}
	\item The first layer is either convolutional or fully connected 
	while all the other layers can be convolutional, fully connected or max-pooling
	\item $(\sigma_1,\ldots,\sigma_L)$ are continuous and non-constant activation functions
    \end{enumerate}
    Then for every layer $1\leq k\leq L$, there exist a set of parameters of the first $k$ layers $(W_l,b_l)_{l=1}^k$ 
    such that it holds $f_k^p(x_i)\neq f_k^q(x_j)$ for every $p,q\in[P_k],i,j\in[N],i\neq j.$
    Moreover, $(W_l,b_l)_{l=1}^k$ can be chosen in such a way that, except for max-pooling layers,
    all the weight matrices $U_l=\mathcal{M}_l(W_l)$
    have full rank for every $1\leq l\leq k.$
\end{lemma}
\begin{proof}
    The high-level idea of the proof is the following: since Assumption \ref{ass:different_patches} holds for the training inputs, 
    one can first pick $(W_1,b_1)$ such that one transports the property of Assumption \ref{ass:different_patches} to the
    feature maps of the training data at the next layer and thus to all higher layers by induction.
    %Then given that Assumption \ref{ass:different_patches} is satisfied at the first layer,
    %the selection of the parameters of the higher layers can be done similarly.
    
    Since our definition of a convolutional layer includes fully connected layer as a special case,
    it is sufficient to prove the result for the general convolutional structure.
    Since the first layer is a convolutional layer by our assumption,
    we denote by $Q=[a^1,\ldots,a^{T_1}]\in\RR^{l_0\times T_1}$
    a matrix that contains the set of convolutional filters of the first layer.
    Note here that there are $T_1$ filters, namely $\Set{a^1,\ldots,a^{T_1}}$,
    where each filter $a^t\in\RR^{l_0}$ for every $t\in[T_1].$
    Let us define the set
    \begin{align*}
	&S\bydef\Setbar{Q\in\RR^{l_0\times T_1}}{ \mathcal{M}_1(Q) \textrm{ has low rank} } \cup \\
	&\bigcup_{\substack{i\neq j\\p,q\in[P_0]\\t,t'\in[T_1]}}
	\Setbar{Q\in\RR^{l_0\times T_1}}{\inner{a^t,x_i^p}-\inner{a^{t'},x_j^q}=0}.
    \end{align*}
    Basically, $S$ is the set of ``true'' parameter matrices of the first layer where
    the corresponding weight matrix $\mathcal{M}_1(Q)$ has low rank or there exists two patches of two different training samples
    that have the same inner product with some corresponding two filters.
    By  Assumption \ref{ass:different_patches} it holds that $x_i^p\neq x_j^q$ for every $p,q\in[P_0],i\neq j$,
    and thus the right hand side in the above formula of $S$ is just the union of a finite number of hyperplanes which has Lebesgue measure zero.
    For the left hand side, it follows from Lemma \ref{lem:WU_measure_zero} that
    the set of $Q$ for which $\mathcal{M}_1(Q)$  does not have full rank has measure zero.
    Thus the left hand side of $S$ is a set of measure zero. 
    Since $S$ is the union of two  measure zero sets, it has also measure zero,
    and thus the complementary set $\RR^{l_0\times T_1}\setminus S$ must be non-empty and we choose $W_1\in\RR^{l_0\times T_1}\setminus S.$
    
    Since $\sigma_1$ is a continuous and non-constant function, there exists an interval $(\mu_1,\mu_2)$ 
    such that $\sigma_1$ is bijective on $(\mu_1,\mu_2)$.
    We select and fix some matrix $Q=[a^1,\ldots,a^{T_1}]\in\RR^{l_0\times T_1}\setminus S$
    %In the rest of the proof, the value of matrix $Q$ is fixed.
    and select some $\beta\in(\mu_1,\mu_2)$.
    Let $\alpha>0$ be a free variable and $W_1=[w_1^1,\ldots,w_1^{T_1}]$ where $w_1^t$ denotes the $t$-th filter of the first layer.
    Let us pick
    \begin{align*}
	w_1^t = \alpha Q_{:t} = \alpha a^t,\quad (b_1)_h=\beta,\quad \forall t\in[T_1], h\in[n_1] .
    \end{align*}
    It follows that $W_1=\alpha Q$ and thus 
    $\mathcal{M}_1(W_1)=\mathcal{M}_1(\alpha Q)=\alpha\mathcal{M}_1(Q)$ as $\mathcal{M}_1$ is a linear map by our definition.
    Since $Q\notin S$ by construction, it holds that $\mathcal{M}_1(W_1)$ has full rank for every $\alpha\neq 0.$
    By Definition \ref{def:conv}, it holds for every $i\in[N],p\in[P_0],t\in[T_1],h=(p-1)T_1 + t$ that
    \begin{align*}
	f_1(x_i)_h=\sigma_1(\inner{w_1^t,x_i^p} + (b_1)_h)=\sigma_1(\alpha\inner{a^t,x_i^p} + \beta) .
    \end{align*}
    Since $\beta\in(\mu_1,\mu_2)$, one can choose a sufficiently small positive $\alpha$ such that
    it holds $\alpha\inner{a,x_i^p} + \beta \in(\mu_1,\mu_2)$ for every $i\in[N],p\in[P_0],t\in[T_1].$ 
    Under this construction, we will show that every entry of $f_1(x_i)$ must be different from every entry of $f_1(x_j)$ for $i\neq j.$
    Indeed, let us compare $f_1(x_i)_h$ and $f_1(x_j)_v$ for some $h=(p-1)T_1+t,v=(q-1)T_1+t'$ and $i\neq j.$
    It holds for sufficient small $\alpha>0$ that
    \begin{align}\label{eq:lem1_lay1}
	&f_1(x_i)_h - f_1(x_j)_v \nonumber \\
	&= \sigma_1\left(\alpha\inner{a^t,x_i^p} + \beta\right) - \sigma_1\left(\alpha\inner{a^{t'},x_j^q} + \beta\right) \nonumber \\
	&\neq 0 
    \end{align}
    where the last inequality follows from three facts.
    First, it holds $\inner{a^t,x_i^p}\neq\inner{a^{t'},x_j^q}$ since $Q\notin S.$
    Second, for the chosen $\alpha$ the values of the arguments of the activation function $\sigma_1$ lie within $(\mu_1,\mu_2).$
    Third, since $\sigma_1$ is bijective on $(\mu_1,\mu_2)$, it maps different inputs to different outputs.
    
    Now, since the entries of $f_1(x_i)$ and that of $f_1(x_j)$ are already pairwise different from each other,
    their corresponding patches must be also different from each other no matter how the patches are organized in the architecture, 
    that is,
    \begin{align*}
	f_1^p(x_i)\neq f_1^q(x_j) \quad\forall p,q\in[P_1], i,j\in[N], i\neq j .
    \end{align*}    
    Now, if the network has only one layer, \ie $L=1$, then we are done.
    Otherwise, we will prove via induction that this property can be translated to any higher layer.
    In particular, suppose that one has already constructed $(W_l,b_l)_{l=1}^k$ so that it holds 
    \begin{align}\label{eq:lem1_layk1}
	f_k(x_i)_h - f_k(x_j)_v \neq 0 \quad\forall h,v\in[n_k], i,j\in[N], i\neq j .
    \end{align}
    This is true for $k=1$ due to \eqref{eq:lem1_lay1}.
    We will show below that \eqref{eq:lem1_layk1} can also hold at layer $k+1$.
    \begin{enumerate}
	\item \underline{Case 1}: Layer $k+1$ is convolutional or fully connected.
	
	Since \eqref{eq:lem1_layk1} holds for $k$ by our induction assumption,
	it must hold that
	\begin{align*}
	    f_k^p(x_i)\neq f_k^q(x_j) \quad\forall p,q\in[P_k], i,j\in[N], i\neq j .
	\end{align*}
	which means Assumption \ref{ass:different_patches} also holds for the set of features at layer $k$.
	Thus one can follows the similar construction as done for layer $1$ above
	by considering the output of layer $k$ as input to layer $k+1$.
	Then one obtains that there exist $(W_{k+1},b_{k+1})$ where $U_{k+1}=\mathcal{M}_{k+1}(W_{k+1})$ has full rank
	so that the similar inequality \eqref{eq:lem1_lay1} now holds for layer $k+1$, 
	which thus implies \eqref{eq:lem1_layk1} holds for $k+1.$
	
	\item \underline{Case 2}: layer $k+1$ is max-pooling
	
	By Definition \ref{def:mp}, it holds $n_{k+1}=P_k$ and one has for every $p\in[P_k]$ 
	\begin{align*}
	    f_{k+1}(x)_p = \max\Big( {(f_k^p(x))}_{1},\ldots,{(f_k^p(x))}_{l_k} \Big) .
	\end{align*}
	Since \eqref{eq:lem1_layk1} holds at layer $k$ by our induction assumption,
	every entry of every patch of $f_k(x_i)$ must be different from every entry of every patch of $f_k(x_j)$ for every $i\neq j$,
	that is, ${(f_k^p(x_i))}_{r} \neq {(f_k^q(x_j))}_{s}$ for every $r,s\in[l_k],p,q\in[P_k],i\neq j.$
	Therefore, their maximum elements cannot be the same, that is,
	\begin{align*}
	    f_{k+1}(x_i)_p \neq f_{k+1}(x_j)_q 
	\end{align*}
	for every $p,q\in[n_{k+1}], i,j\in[N], i\neq j,$
	which proves \eqref{eq:lem1_layk1} for layer $k+1.$
    \end{enumerate}
    So far, we have proved that \eqref{eq:lem1_layk1} holds for every $1\leq k\leq L$.
    Thus it follows that for every layer $k$, 
    there exists a set of parameters of the first $k$ layers
    for which the patches at layer $k$  of different training samples
    are pairwise different from each other, that is,
    $f_k^p(x_i)\neq f_k^q(x_j)$ for every $p,q\in[P_k],i\neq j.$
    Moreover, except for max-pooling layers, all the weight matrices up to layer $k$ have been chosen to have full rank.
\end{proof}
\paragraph{Proof of Theorem \ref{theo:linear_independence}}
% \begin{proof}
    Let $A=F_k=[f_{k}(x_1)^T,\ldots,f_{k}(x_N)^T]\in\RR^{N\times n_k}.$
    Since our definition of a convolutional layer includes fully connected layer as a special case,
    it is sufficient to prove the result for convolutional structure.
    By Theorem's assumption, layer $k$ is convolutional and thus it holds by Definition \ref{def:conv} that
    $$A_{ij} = {f_k(x_i)}_j = \sigma\Big( \inner{w_k^t,f_{k-1}^p(x_i)} + (b_k)_j \Big)$$
    for every $i\in[N],t\in[T_k],p\in[P_{k-1}]$ and $j=(p,t)\bydef(p-1)T_k+t\in[n_k].$
    
    In the following, we show that there exists a set of parameters of the network
    such that $\rank(A)=N$ and all the weight matrices $U_l=\mathcal{M}_l(W_l)$ have full rank.

    	    First, one observes that the subnetwork consisting of all the layers from the input layer till layer $k-1$ 
    	    satisfies the conditions of Lemma \ref{lem:transport_different_patches}.
	    Thus by applying Lemma \ref{lem:transport_different_patches} to this subnetwork,
	    one obtains that there exist $(W_l,b_l)_{l=1}^{k-1}$
	    for which all the matrices $(U_l)_{l=1}^{k-1}$, except for max-pooling layers, have full rank
	    and it holds that $f_{k-1}^p(x_i)\neq f_{k-1}^q(x_j)$ for every $p,q\in[P_{k-1}],i\neq j.$
	    The main idea now is to fix the parameters of these layers and 
	    pick $(W_k,b_k)$ such that $U_k=\mathcal{M}_k(W_k)$ has full rank and it holds $\rank(A)=N$.
	    Let us define the set
	    \begin{align*}
		S\!=\!\!\!\!\bigcup_{\substack{i\neq j\\p\in[P_{k-1}]}}
		\Setbar{a\in\RR^{l_{k-1}}}{\inner{a,f_{k-1}^p(x_i)-f_{k-1}^p(x_j)}=0} .
	    \end{align*}
	    From the above construction, it holds that $f_{k-1}^p(x_i)\neq f_{k-1}^p(x_j)$ for every $p\in[P_{k-1}],i\neq j$, 
	    and thus $S$ is the union of a finite number of hyperplanes which thus has measure zero.
	    Let us denote by $Q=[a^1,\ldots,a^{T_k}]\in\RR^{l_{k-1}\times T_k}$ a parameter matrix
	    that contains all the convolutional filters of layer $k$ in its columns.
	    Pick $a^t\in\RR^{l_{k-1}}\setminus S$ for every $t\in[T_k]$,
	    so that it holds that $U_k=\mathcal{M}_k(Q)$ has full rank.
	    Note here that such matrix $Q$ always exists. 
	    Indeed, $Q$ is chosen from a positive measure set as its columns (\ie $a^t$) are picked from a positive measure set.
	    Moreover, the set of matrices $Q$ for which $\mathcal{M}_k(Q)$ has low rank has just measure zero due to Lemma \ref{lem:WU_measure_zero}.
	    Thus there always exists at least one matrix $Q$ so that all of its columns do not belong to $S$
	    and that $\mathcal{M}_k(Q)$ has full rank.
	    In the rest of the proof, the value of matrix $Q$ is fixed.
	    Let $\alpha\in\RR$ be a free parameter.
	    Since $\sigma_k$ is a continuous and non-constant function, there exist a $\beta\in\RR$ 
	    such that $\sigma_k(\beta)\neq 0.$
	    Let the value of $\beta$ be fixed as well.
	    We construct the convolutional filters $W_k=[w_k^1,\ldots,w_k^{T_k}]$ and the biases $b_k\in\RR^{n_k}$ of layer $k$ as follows.
	    For every $p\in[P_{k-1}],t\in[T_k],j=(p,t)$, we define
	    \begin{align*}
		&w_k^t=- \alpha a^t,\\ 
		&(b_k)_j=\alpha\inner{a^t,f_{k-1}^p(x_j)} + \beta .
	    \end{align*}
	    It follows that $W_k=-\alpha Q$ and thus $U_k=\mathcal{M}_k(W_k)=-\alpha\mathcal{M}_k(Q)$
	    as $\mathcal{M}_k$ is a linear map.
	    Moreover, since $\mathcal{M}_k(Q)$ has full rank by construction, 
	    it holds that $U_k$ has full rank for every $\alpha\neq 0.$
	    As $\alpha$ varies,  we get a family of matrices $A(\alpha)\in\RR^{N\times n_k}$ where
	    it holds for every $i\in[N],j=(p,t)\in[n_k]$ that
	    \begin{align}\label{eq:Aij}
		  A(\alpha)_{ij} 
		  &= \sigma_k\Big( \inner{w_k^t,f_{k-1}^p(x_i)} + (b_k)_j \Big) \nonumber \\
		  &= \sigma_k\Big( \alpha \inner{a^t,f_{k-1}^p(x_j) - f_{k-1}^p(x_i)} + \beta \Big) .
	    \end{align} 
	    Note that each row of $A(\alpha)$ corresponds to one training sample and that
	    permutations of the rows of $A(\alpha)$ do not change the rank of $A(\alpha)$.
	    We construct a permutation $\gamma$ of $\Set{1,2,\ldots,N}$ as follows.
	    For every $j=1,2,\ldots,N$, let $(p,t)$ be the tuple determined by $j$ (the inverse transformation for given $j \in [n_k]$ is $p=\left\lceil \frac{j}{T_k}\right\rceil$ and $t=j -\Big(\left\lceil \frac{j}{T_k}\right\rceil -1\Big)T_k$) and define
		\begin{align*}
		    \gamma(j) = \argmin_{i\in\Set{1,2,\ldots,N}
		    \,\setminus\,\Set{\gamma(1),\ldots,\gamma(j-1)}} 
		    \inner{a^t, f_{k-1}^p(x_i)} .
		\end{align*}
	    Note that $\gamma(j)$ is unique for every $1\leq j\leq N$ since $a^t\notin S.$ 
	    It is clear that $\gamma$ constructed as above is a permutation of $\Set{1,2,\ldots,N}$
	    since every time a different element is taken from the index set $[N]$.
	    From the definition of $\gamma(j)$, it holds that for every $j=(p,t)\in[N],i\in[N],i>j$ that
	    \begin{align*}
		\inner{a^t, f_{k-1}^p(x_{\gamma(j)})} < \inner{a^t, f_{k-1}^p(x_{\gamma(i)})} .
	    \end{align*}
% 	    The intuitive reason is because $\gamma(j)$ is constructed before $\gamma(i)$ 
% 	    and thus $\gamma(i)$ is included in the constraint set of the optimization problem 
% 	    where $\gamma(j)$ is the unique minimizer.
            We can relabel the training data according to the permutation so that w.l.o.g we can assume that $\gamma$ is the identity permutation, 
            that is, $\gamma(j)=j$ for every $j\in[N],$
	    in which case it holds for every $j=(p,t)\in[N],i\in[N],i>j$ that
	    \begin{align}\label{eq:af}
		\inner{a^t, f_{k-1}^p(x_j)} < \inner{a^t, f_{k-1}^p(x_i)} .
	    \end{align}
	    Under the above construction of $(W_l,b_l)_{l=1}^{k}$, 
	    we are ready to show that there exist $\alpha$ for which $\rank(A(\alpha))=N.$
	    Since $\sigma_k$ satisfies Assumption \ref{ass:activation}, we consider the following cases.
	    \begin{enumerate}
		\item \underline{Case 1}: 
		There are finite constants $\mu_{+},\mu_{-}\in\RR$ s.t. 
		$\lim\limits_{t\to-\infty} \sigma_k(t)=\mu_{-}$ and $\lim\limits_{t\to\infty} \sigma_k(t)=\mu_{+}$  
		and $\mu_{+} \mu_{-}=0.$ \\
		Let us consider the first case where $\mu_{-}=0.$
		From \eqref{eq:Aij} and \eqref{eq:af} one obtains
		\begin{align}\label{eq:A_structure}
		    \lim\limits_{\alpha\to+\infty} A(\alpha)_{ij} = 
		    \begin{cases}
			\sigma_k(\beta) & j=i \\ 
			\mu_{-}=0 & i>j \\ 
			\eta_{ij} & i<j 
		    \end{cases}
		\end{align}
		where $\eta_{ij}$ is given for every $i<j$ where $j=(p,t)$ as
		\begin{align*}
		    \eta_{ij}=
		    \begin{cases}
			\mu_{-}, & \inner{a^t, f_{k-1}^p(x_j)-f_{k-1}^p(x_i)} < 0\\ 
			\mu_{+}, & \inner{a^t, f_{k-1}^p(x_j)-f_{k-1}^p(x_i)} > 0 
		    \end{cases}
		\end{align*}
		Note that $\eta_{ij}$ cannot be zero for $i\neq j$ because $a^t\notin S.$
		In the following, let us denote $A(\alpha)_{1:N,1:N}$ as a sub-matrix of $A(\alpha)$
		that consists of the first $N$ rows and columns.
		By the Leibniz-formula one has
		\begin{align}\label{eq:leibniz}
		    &\det(A(\alpha)_{1:N,1:N}) \nonumber \\
		    &=\sigma_k(\beta)^{N} + \sum_{\pi \in S_N \backslash \{\gamma\} } 
		    \mathrm{sign}(\pi) \prod_{j=1}^{N} A(\alpha)_{\pi(j) j}
		  \end{align}
		where $S_N$ is the set of all $N!$ permutations of the set $\{1,\ldots,N\}$
		and $\gamma$ is the identity permutation.
		Now, one observes that for every permutation $\pi\neq \gamma$,
		there always exists at least one component $j$ where $\pi(j)>j$ 
		in which case it follows from \eqref{eq:A_structure} that
		\begin{align*}
		    \lim\limits_{\alpha\to\infty} \prod_{j=1}^{N} A(\alpha)_{\pi(j) j} = 0 .
		\end{align*}
		Since there are only finitely many such terms in \eqref{eq:leibniz}, one obtains
		\begin{align*}
		    \lim\limits_{\alpha\to\infty} \det(A(\alpha)_{1:N,1:N}) = \sigma_k(\beta)^N \neq 0
		\end{align*}
		where the last inequality follows from our choice of $\beta$.
		Since $\det(A(\alpha)_{1:N,1:N})$ is a continuous function of $\alpha$,
		there exists $\alpha_0\in\RR$ such that for every $\alpha\geq\alpha_0$ 
		it holds $\det(A(\alpha)_{1:N,1:N})\neq 0$ and thus $\rank(A(\alpha)_{1:N,1:N})=N$ 
		which further implies $\rank(A(\alpha))=N.$
		Thus the corresponding set of feature vectors 
		$\Set{f_k(x_1),\ldots,f_k(x_N)}$ are linearly independent.
		
		For the case where $\mu_{+}=0$, one can argue similarly.
		The only difference is that one considers now the limit for $\alpha\to-\infty.$
		In particular, \eqref{eq:Aij} and \eqref{eq:af} lead to
		\begin{align*}
		    \lim\limits_{\alpha\to-\infty} A(\alpha)_{ij} = 
		    \begin{cases}
			\sigma_k(\beta) & i=j\\
			\mu_{+}=0 & i>j \\
			\eta_{ij}=0 & i<j .
		    \end{cases}
		\end{align*}
		For every permutation $\pi\neq \gamma$ there always exists at least one component $j$ where $\pi(j)>j$,
		in which case it holds that
		\begin{align*}
		    \lim\limits_{\alpha\to-\infty} \prod_{j=1}^{N} A(\alpha)_{\pi(j) j} = 0 .
		\end{align*}
		and thus it follows from the Leibniz formula that
		\begin{align*}
		    \lim\limits_{\alpha\to-\infty} \det(A(\alpha)_{1:N,1:N}) = \sigma_k(\beta)^N \neq 0 .
		\end{align*}
		Since $\det(A(\alpha)_{1:N,1:N})$ is a continuous function of $\alpha$,
		there exists $\alpha_0\in\RR$ such that for every $\alpha\leq\alpha_0$ 
		it holds $\det(A(\alpha)_{1:N,1:N})\neq 0$ and thus $\rank(A(\alpha)_{1:N,1:N})=N$ 
		which further implies $\rank(A(\alpha))=N.$
		Thus the set of feature vectors at layer $k$ are linearly independent.

		\item \underline{Case 2}: There are positive constants $\rho_1,\rho_2,\rho_3,\rho_4$ 
		s.t. $|\sigma_k(t)|\leq \rho_1 e^{\rho_2 t}$ for $t< 0$ 
		and $|\sigma_k(t)|\leq \rho_3 t + \rho_4 $ for $t\geq 0.$\\
		Our proof strategy is essentially similar to the previous case.
		Indeed, for every permutation $\pi\neq \gamma$ 
		there always exist at least one component $j=(p,t)\in[N]$ where $\pi(j)>j$ in which case 
		$\delta_j\bydef\inner{a^t,f_{k-1}^p(x_j) - f_{k-1}^p(x_{\pi(j)})} < 0$
		due to \eqref{eq:af}.
		For sufficiently large $\alpha>0$, it holds that $\alpha \delta_j + \beta<0$ 
		and thus one obtains from \eqref{eq:Aij} that
		\begin{align*}
		    |A(\alpha)_{\pi(j) j}| = 
		    |\sigma_k(\alpha\delta_j + \beta)| 
		    \leq \rho_1 e^{\rho_2\beta} e^{-\alpha \rho_2 |\delta_j|} .
		\end{align*}
		If $\pi(j)=j$ then $|A(\alpha)_{\pi(j) j}|=|\sigma_k(\beta)|$ which is a constant.
		For $\pi(j)<j$, one notices that 
		$\delta_j\bydef\inner{a^t,f_{k-1}^p(x_j) - f_{k-1}^p(x_{\pi(j)})}$
		can only be either positive or negative as $a^t\notin S.$
		In this case, if $\delta_j<0$ then it can be bounded by the similar exponential term as above
		for sufficiently large $\alpha$.
		Otherwise it holds $\alpha \delta_j + \beta>0$ 
		for sufficiently large $\alpha>0$ and we get
		\begin{align*} 
		    |A(\alpha)_{\pi(j) j}|
		    =|\sigma(\alpha \delta_j + \beta)| 
		    \leq \rho_3 \delta_j \alpha + \rho_3 \beta +\rho_4 .
		\end{align*}
		Overall, for sufficiently large $\alpha>0$, 
		there must exist positive constants $P,Q,R,S,T$ 
		such that it holds for every $\pi\in S_N\setminus\Set{\gamma}$ that		
		\[  \Big|\prod_{j=1}^{N} A(\alpha)_{\pi(j) j}\Big| \; \leq \; R (P \alpha + Q)^{S} e^{-\alpha T}.\]
		The upper bound goes to zero as $\alpha$ goes to $\infty.$ 
		This combined with the Leibniz formula from \eqref{eq:leibniz}, we get
		$\lim_{\alpha \rightarrow \infty}  \det(A(\alpha)_{1:N,1:N}) = \sigma_k(\beta)^{N} \neq 0.$
		Since $\det(A(\alpha)_{1:N,1:N})$ is a continuous function of $\alpha$,
		there exists $\alpha_0\in\RR$ such that for every $\alpha\geq\alpha_0$ 
		it holds $\det(A(\alpha)_{1:N,1:N})\neq 0$ and thus $\rank(A(\alpha)_{1:N,1:N})=N$ 
		which implies $\rank(A(\alpha))=N.$
		Thus the set of feature vectors at layer $k$ are linearly independent.
	    \end{enumerate}
    Overall, we have shown that there always exist $(W_l,b_l)_{l=1}^k$ 
    such that the set of feature vectors $\Set{f_k(x_1),\ldots,f_k(x_N)}$ at layer $k$ are linearly independent.
    Moreover, $(W_l,b_l)_{l=1}^k$ have been chosen so that all the weight matrices $U_l=\mathcal{M}_l(W_l)$,
    except for max-pooling layers, have full rank for every $1\leq l\leq k.$
% \end{proof}

\section{Proof of Theorem \ref{theo:how_often}}
% \begin{proof}
    Any linear function is real analytic and 
    the set of real analytic functions is closed under addition, multiplication and composition, 
    see e.g. Prop. 2.2.2 and Prop. 2.2.8 in \cite{KraPar2002}.
    As we assume that all the activation functions of the first $k$ layers are real analytic, 
    we get that the function $f_k$ 
    is a real analytic function of $(W_l,b_l)_{l=1}^k$ as it is the composition of real analytic functions. 
    Now, we recall from our definition that $F_k=[f_k(x_1),\ldots,f_k(x_N)]^T\in\RR^{N\times n_k}$
    is the output matrix at layer $k$ for all training samples.
    One observes that the set of low rank matrices $F_k$ can be characterized by a system of equations 
    such that all the $\binom{n_k}{N}$ determinants of all $N \times N$ sub-matrices of $F_k$ are zero. 
    As the determinant is a polynomial in the entries of the matrix
    and thus an analytic function of the entries and composition of analytic functions are again analytic, 
    we conclude that each determinant is an analytic function of the network parameters of the first $k$ layers. 
    By Theorem \ref{theo:linear_independence}, 
    there exists at least one set of parameters of the first $k$ layers 
    such that one of these determinant functions
    is not identically zero and thus by Lemma \ref{lem:zeros_of_analytic},
    the set of network parameters where this determinant is zero has Lebesgue measure zero.
    But as all submatrices need to have low rank in order that $\rank(F_k)<N$, 
    it follows that the set of parameters where $\rank(F_k)<N$ has Lebesgue measure zero. 
% \end{proof}

\section{Proof of Corollary \ref{cor:finite_expressivity}}
% \begin{proof}    
    Since the network satisfies the conditions of Theorem \ref{theo:linear_independence} for $k=L-1$,
    there exists a set of parameters $(W_l,b_l)_{l=1}^{L-1}$ such that $\rank(F_{L-1})=N.$
    Let $F_L=[f_L(x_1),\ldots,f_L(x_N)]^T\in\RR^{N}$ then it follows that $F_L=F_{L-1}\lambda.$
    Pick $\lambda = F_{L-1}^T(F_{L-1}F_{L-1}^T)^{-1}y$
    then it holds $F_L = F_{L-1} \lambda=y.$
% \end{proof}

\section{Proof of Lemma \ref{lem:S_k}}
% \begin{proof}
    One can see that 
    \begin{align*}
	&\mathcal{P}\setminus S_k\subseteq \Setbar{(W_l,b_l)_{l=1}^L}{\rank(F_k) < N} \cup \\
	&\Setbar{(W_l,b_l)_{l=1}^L}{ U_l \textrm{ has low rank for some layer } l }.
    \end{align*}
    By Theorem \ref{theo:how_often}, 
    it holds that the set $\Setbar{(W_l,b_l)_{l=1}^L}{\rank(F_k) < N}$ has Lebesgue measure zero.
    Moreover, it follows from Lemma \ref{lem:WU_measure_zero} that 
    the set $\Setbar{(W_l,b_l)_{l=1}^L}{ U_l \textrm{ has low rank for some layer } l }$ also has measure zero.
    Thus, $\mathcal{P}\setminus S_k$ has Lebesgue measure zero.
% \end{proof}

\section{Proof of Lemma \ref{lem:bounds}}
To prove Lemma \ref{lem:bounds}, we first derive standard backpropagation in Lemma \ref{lem:grad}.
In the following we use the Hadamard product $\circ$, 
which for $A,B \in \RR^{m \times n}$ is defined as $A \circ B \in \RR^{m \times n}$ with $(A \circ B)_{ij}=A_{ij}B_{ij}$.
Let $\delta_{kj}(x_i)=\frac{\partial\Phi}{\partial g_{kj}(x_i)}$
be the derivative of $\Phi$ w.r.t. the value of unit $j$ at layer $k$ evaluated at a single sample $x_i$.
We arrange these vectors for all training samples into a single matrix
$\Delta_k=[\delta_{k:}(x_1), \ldots, \delta_{k:}(x_N)]^T\in\RR^{N\times n_k}.$ The following lemma is
a slight modification of Lemma 2.1 in \citep{Quynh2017} for which we provide the proof for completeness.
\begin{lemma}\label{lem:grad}
    Given some hidden layer $1\leq k\leq L-1$. 
    Let $(\sigma_{k+1},\ldots,\sigma_{L-1})$ be differentiable.
    Then the following hold:
    \begin{enumerate}
	\item $\Delta_l=
	    \begin{cases}
		F_L-Y, &l=L\\
		(\Delta_{l+1}U_{l+1}^T) \circ \sigma_l'(G_{l}), &k+1\leq l\leq L-1
	    \end{cases}$
	\item $\nabla_{U_l}\Phi=F_{l-1}^T \Delta_l,\,\forall\, k+1\leq l\leq L$
    \end{enumerate}
\end{lemma}
\begin{proof}
\begin{enumerate}
	\item By definition, it holds for every $i\in[N],j\in[n_L]$ that
	\begin{align*}
	    (\Delta_L)_{ij}&=\delta_{Lj}(x_i)
	    =\frac{\partial\Phi}{\partial g_{Lj}(x_i)} 
	    =f_{Lj}(x_i)-y_{ij}
	\end{align*}
	and hence, $\Delta_L=F_L-Y.$
	
	For every $k+1\leq l\leq L-1$, the output of the network for a single training sample can be written as 
	the composition of differentiable functions (\ie the outputs of all layers from $l+1$ till the output layer),
	and thus the chain rule yields for every $i\in[N],j\in[n_l]$ that
	\begin{align*}
	    (\Delta_l)_{ij} 
	    &=\delta_{lj}(x_i) \\
	    &=\frac{\partial\Phi}{\partial g_{lj}(x_i)} \\
	    &=\sum_{s=1}^{n_{l+1}} \frac{\partial\Phi}{\partial g_{(l+1)s}(x_i)} \frac{\partial g_{(l+1)s}(x_i)}{\partial f_{lj}(x_i)}  \frac{\partial f_{lj}(x_i)}{\partial g_{lj}(x_i)}  \\
	    &=\sum_{s=1}^{n_{l+1}} \delta_{(l+1)s} (x_i) {(U_{l+1})}_{js} \sigma'(g_{lj}(x_i)) \\
	    &=\sum_{s=1}^{n_{l+1}} {(\Delta_{(l+1)})}_{is} {(U_{l+1})}_{sj}^T \sigma'({(G_l)}_{ij}) 
	\end{align*}
	and hence $\Delta_l=(\Delta_{l+1}U_{l+1}^T) \circ \sigma'(G_{l}).$
	
	\item For every $k+1\leq l\leq L-1,r\in[n_{l-1}],s\in[n_l],$ one has
	\begin{align*}
	    \frac{\partial\Phi}{\partial (U_l)_{rs}} 
	    &= \sum_{i=1}^N \frac{\partial\Phi}{\partial g_{ls}(x_i)} \frac{\partial g_{ls}(x_i)}{\partial (U_l)_{rs}} \\
	    &=\sum_{i=1}^N \delta_{ls}(x_i) f_{(l-1)r}(x_i) \\
	    &=\sum_{i=1}^N (F_{l-1}^T)_{ri} {(\Delta_l)}_{is} \\
	    &=\big( F_{l-1}^T \Delta_l\big)_{rs}
	\end{align*}
	and hence $\nabla_{U_l}\Phi={F_{l-1}^T}\Delta_l.$
\end{enumerate} 
\end{proof}
The following straightforward inequalities are also helpful to prove Lemma \ref{lem:bounds}.
Let $\evmin{\cdot}$ and $\evmax{\cdot}$ denotes the smallest and largest eigenvalue of a matrix.
\begin{lemma}\label{lem:Ax_2}
    Let $A\in\RR^{m\times n}$ with $m\geq n$.
    Then it holds $\svmax{A}\norm{x}_2\geq\norm{Ax}_2\geq\svmin{A}\norm{x}_2$ for every $x\in\RR^n.$
\end{lemma}
\begin{proof}
    Since $m\geq n$, it holds that 
    $\svmin{A}=\sqrt{\evmin{A^TA}}=\sqrt{\min\frac{x^TA^TAx}{x^Tx}}=\min\frac{\norm{Ax}_2}{\norm{x}_2}$
    and thus $\svmin{A}\leq\frac{\norm{Ax}_2}{\norm{x}_2}$ for every $x\in\RR^n.$
    Similarly, it holds
    $\svmax{A}=\sqrt{\evmax{A^TA}}=\sqrt{\max\frac{x^TA^TAx}{x^Tx}}=\max\frac{\norm{Ax}_2}{\norm{x}_2}$
    and thus $\svmax{A}\geq\frac{\norm{Ax}_2}{\norm{x}_2}$ for every $x\in\RR^n.$
\end{proof}
\begin{lemma}\label{lem:AB_F}
    Let $A\in\RR^{m\times n}, B\in\RR^{n\times p}$ with $m\geq n$.
    Then it holds $\svmax{A}\Fnorm{B}\geq\Fnorm{AB}\geq\svmin{A}\Fnorm{B}.$
\end{lemma}
\begin{proof}
    Since $m\geq n$, it holds that $\evmin{A^TA}=\svmin{A}^2$ and $\evmax{A^TA}=\svmax{A}^2.$
    Thus we have $\Fnorm{AB}^2=\tr(B^TA^TAB)\geq\evmin{A^TA}\tr(B^T B)=\svmin{A}^2\Fnorm{B}^2.$
    Similarly, it holds $\Fnorm{AB}^2=\tr(B^TA^TAB)\leq\evmax{A^TA}\tr(B^T B)=\svmax{A}^2\Fnorm{B}^2.$
\end{proof}

\paragraph{Proof of Lemma \ref{lem:bounds}}
% \begin{proof}
    We first prove the lower bound.
    Let $\Id_{m}$ denotes an $m$-by-$m$ identity matrix and $\otimes$ the Kronecker product.
    From Lemma \ref{lem:grad} it holds $\nabla_{U_{k+1}}\Phi=F_L^T\Delta_{k+1}$ and thus 
    \begin{align*}
	\vec(\nabla_{U_{k+1}}\Phi)=(\Id_{n_{k+1}}\otimes F_k^T)\vec(\Delta_{k+1}) .
    \end{align*}
    It follows that
    \begin{align}\label{eq:lowerbound_1}
	\Fnorm{\nabla_{U_{k+1}}\Phi} 
	&=  \norm{(\Id_{n_{k+1}}\otimes F_k^T) \vec(\Delta_{k+1})}_2 \nonumber \\
	&\geq  \svmin{F_k} \norm{\vec(\Delta_{k+1})}_2 \nonumber \\
	&=  \svmin{F_k} \norm{\Delta_{k+1}}_F 
    \end{align}
    where the inequality follows from Lemma \ref{lem:Ax_2} 
    for the matrix $(\Id_{n_{k+1}}\otimes F_k^T)\in\RR^{n_k n_{k+1}\times N n_{k+1}}$ with $n_k\geq N$ 
    by Assumption \ref{ass:architecture}.
    Using Lemma \ref{lem:grad} again, one has
    \begin{align*}
	\norm{\Delta_{k+1}}_F
	&= \Fnorm{(\Delta_{k+2}U_{k+2}^T) \circ \sigma_{k+1}'(G_{k+1})} \\
	&\geq \minnorm{\sigma_{k+1}'(G_{k+1})} \Fnorm{\Delta_{k+2}U_{k+2}^T} \\
	&= \minnorm{\sigma_{k+1}'(G_{k+1})} \Fnorm{U_{k+2} \Delta_{k+2}^T}  \\
	&\geq \minnorm{\sigma_{k+1}'(G_{k+1})} \svmin{U_{k+2}}  \Fnorm{\Delta_{k+2}} 
    \end{align*}
    where the last inequality follows from Lemma \ref{lem:AB_F} for the matrices
    $U_{k+2}\in\RR^{n_{k+1}\times n_{k+2}}$ and $\Delta_{k+2}^T$  with $n_{k+1}\geq n_{k+2}$ by Assumption \ref{ass:architecture}.
    By repeating this for $\norm{\Delta_{k+2}}_F,\ldots,\norm{\Delta_{L-1}}_F$, one gets
    \begin{align}\label{eq:lowerbound_2}
	\Fnorm{\Delta_{k+1}} 
	&\geq \Big( \prod_{l=k+1}^{L-1} \minnorm{\sigma_l'(G_l)} \svmin{U_{l+1}}  \Big) \Fnorm{\Delta_L}  \nonumber \\
	&=    \Big( \prod_{l=k+1}^{L-1} \minnorm{\sigma_l'(G_l)} \svmin{U_{l+1}}  \Big) \Fnorm{F_L-Y}
    \end{align}
    From \eqref{eq:lowerbound_1}, \eqref{eq:lowerbound_2},
    one obtains 
    \begin{align*}
	\Fnorm{\nabla_{U_{k+1}}\Phi} 
	\geq  \svmin{F_k} 
	\Big( \prod_{l=k+1}^{L-1} \minnorm{\sigma_l'(G_l)} \svmin{U_{l+1}} \Big) \\
	\Fnorm{F_L-Y} 
    \end{align*}
    which proves the lower bound.
    
    %================================================
    The proof for upper bound is similar. Indeed one has
    \begin{align}\label{eq:upperbound_1}
	\Fnorm{\nabla_{U_{k+1}}\Phi} 
	&= \norm{(\Id_{n_{k+1}}\otimes F_k^T) \vec(\Delta_{k+1})}_2 \nonumber \\
	&\leq \svmax{F_k} \norm{\vec(\Delta_{k+1})}_2  \nonumber \\
	&= \svmax{F_k} \norm{\Delta_{k+1}}_F 
    \end{align}
    where the inequality follows from Lemma \ref{lem:Ax_2}
    Now, one has 
    \begin{align*}
	\norm{\Delta_{k+1}}_F
	&= \norm{(\Delta_{k+2}U_{k+2}^T) \circ \sigma_{k+1}'(G_{k+1})}_F \\
	&\leq \maxnorm{\sigma_{k+1}'(G_{k+1})} \Fnorm{\Delta_{k+2}U_{k+2}^T} \\
	&= \maxnorm{\sigma_{k+1}'(G_{k+1})} \Fnorm{U_{k+2}\Delta_{k+2}^T} \\
	&\leq \maxnorm{\sigma_{k+1}'(G_{k+1})} \svmax{U_{k+2}} \norm{\Delta_{k+2}}_F 
    \end{align*}
    where the last inequality follows from Lemma \ref{lem:AB_F}.
    By repeating this chain of inequalities for $\Fnorm{\Delta_{k+2}},\ldots,\Fnorm{\Delta_{L-1}}$, one  obtains:
    \begin{align}\label{eq:upperbound_2}
	\norm{\Delta_{k+1}}_F 
	&\leq \Big( \prod_{l=k+1}^{L-1} \maxnorm{\sigma_l'(G_l)} \svmax{U_{l+1}} \Big) \norm{\Delta_L}_F \nonumber \\
	&= \Big( \prod_{l=k+1}^{L-1} \maxnorm{\sigma_l'(G_l)} \svmax{U_{l+1}} \Big) \Fnorm{F_L-Y} . 
    \end{align}
    From \eqref{eq:upperbound_1}, \eqref{eq:upperbound_2},
    one obtains that 
    \begin{align*}
	\Fnorm{\nabla_{U_{k+1}}\Phi} 
	\leq \svmax{F_k}
	\Big( \prod_{l=k+1}^{L-1} \maxnorm{\sigma_l'(G_l)} \svmax{U_{l+1}}  \Big) \\
	\Fnorm{F_L-Y} 
    \end{align*}
    which proves the upper bound.
% \end{proof}

\section{Proof of Theorem \ref{theo:loss_surface_special_case}}
\begin{enumerate}
     \item 
    Since layer $k+1$ is fully connected, 
    it holds at every critical point in $S_k$ that $\nabla_{W_{k+1}}\Phi=0=\nabla_{U_{k+1}}\Phi.$
    This combined with Theorem \ref{theo:loss_surface} yields the result.
    \item 
	One basically needs to show that there exist $(W_l,b_l)_{l=1}^L$ such that it holds: 
	$\Phi\Big( (W_l,b_l)_{l=1}^L \Big)=0, \rank(F_k)=N$ and 
	$U_l=\mathcal{M}_l(W_l)$ has full rank $\forall l\in[k+2,L]$
	Note that the last two conditions are fulfilled by the fact that $(W_l,b_l)_{l=1}^L\in S_k$.
	    
	By Assumption \ref{ass:architecture}, 
	the subnetwork consisting of the first $k$ layers satisfies the condition of Theorem \ref{theo:linear_independence}.
	Thus by applying Theorem \ref{theo:linear_independence} to this subnetwork,
	one obtains that there exist $(W_l,b_l)_{l=1}^k$ such that $\rank(F_k)=N$.
	In the following, we fix these layers and show how to pick $(W_l,b_l)_{l=k+1}^L$ such that $F_L=Y.$
	The main idea now is to make the output of all the training samples of the same class
	become identical at layer $k+1$ and thus they will have the same network output.
	Since there are only $m$ classes, there would be only $m$ distinct outputs for all the training samples at layer $L-1$.
	Thus if one can make these $m$ distinct outputs become linearly independent at layer $L-1$
	then there always exists a weight matrix $W_L$ that realizes the target output $Y$
	as the last layer is just a linear map by assumption.
	Moreover, we will show that, except for max-pooling layers, all the parameters of other layers in the network 
	can be chosen in such a way that all the weight matrices $(U_l)_{l=k+2}^L$ achieve full rank.
	Our proof details are as follows. 
	    
	    \underline{Case 1}: $k=L-1$\\
	    It holds $\rank(F_{L-1})=N.$ 
	    Pick $b_L=0$ and $W_L=F_{L-1}^T(F_{L-1}F_{L-1}^T)^{-1} Y$.
	    Since the output layer is fully connected, it follows from Definition \ref{def:fc} that 
	    $F_L = F_{L-1}W_L + \ones_Nb_L^T = Y .$
	    Since $\rank(F_k)=N$ and the full rank condition on $(W_l)_{l=k+2}^L$ is not active when $k=L-1$,
	    it holds that $(W_l,b_l)_{l=1}^L\in S_k$ which finishes the proof.
	    \\

	    \underline{Case 2}: $k=L-2$\\
	    It holds $\rank(F_{L-2})=N.$
	    Let $A\in\RR^{m\times n_{L-1}}$ be a full row rank matrix such that $A_{ij}\in\range(\sigma_{L-1}).$
	    Note that $n_{L-1}\geq n_L=m$ due to Assumption \ref{ass:architecture}.
	    Let $D\in\RR^{N\times n_{L-1}}$ be a matrix satisfying
	    $D_{i:}=A_{j:}$ whenever $x_i$ belongs to class $j$ for every $i\in[N],j\in[m].$ 
	    By construction, $F_{L-2}$ has full row rank, thus we can pick
	    $b_{L-1} = 0, W_{L-1}=F_{L-2}^T(F_{L-2}F_{L-2}^T)^{-1} \sigma_{L-1}^{-1}(D).$
	    Since layer $L-1$ is fully connected by assumption, it follows from Definition \ref{def:fc} that
	    $F_{L-1}=\sigma_{L-1}(F_{L-2}W_{L-1} + \ones_Nb_{L-1}^T) = D$
	    and thus $(F_{L-1})_{i:}=D_{i:}=A_{j:}$ whenever $x_i$ belongs to class $j$. 
	    
	    So far, our construction of the first $L-1$ layers has led to the fact that all the training samples belonging to 
	    the same class will have identical output at layer $L-1.$
	    Since $A$ has full row rank by construction, we can pick for the last layer
	    $b_L=0, W_L = A^T(AA^T)^{-1} Z$
	    where $Z\in\RR^{m\times m}$ is our class embedding matrix with $\rank(Z)=m.$
	    One can easily check that $AW_L = Z$ and that $F_L = F_{L-1}W_L + \ones_Nb_L^T= F_{L-1}W_L$
	    where the later follows from Definition \ref{def:fc} as the output layer is fully connected.
	    Now one can verify that $F_L=Y.$ 
	    Indeed, whenever $x_i$ belongs to class $j$ one has
	    \begin{align*}
		(F_L)_{i:} =(F_{L-1})_{i:}^T W_L =(A)_{j:}^T W_L  =Z_{j:} =Y_{i:} .
	    \end{align*}
	    Moreover, since $\rank(F_k)=N$ and $\rank(W_L)=\rank(A^T(AA^T)^{-1} Z)=m$, 
	    it holds that $(W_l,b_l)_{l=1}^L \in S_k.$ 
	    Therefore, there exists $(W_l,b_l)_{l=1}^L\in  S_k$ with $\Phi\Big((W_l,b_l)_{l=1}^L\Big)=0.$
	    \\

	    \underline{Case 3}: $k\leq L-3$\\
	    It holds $\rank(F_k)=N.$
	    Let $E\in\RR^{m\times n_{k+1}}$ be any matrix such that $E_{ij}\in\range(\sigma_{k+1})$
	    and $E_{ip}\neq E_{jq}$ for every $1\leq i\neq j\leq N,1\leq p,q\leq n_{k+1}.$
	    Let $D\in\RR^{N\times n_{k+1}}$ satisfies $D_{i:}=E_{j:}$ for every $x_i$ from class $j.$ 
	    Pick $b_{k+1} = 0, W_{k+1}=F_k^T(F_kF_k^T)^{-1} \sigma_{k+1}^{-1}(D).$
	    Note that the matrix is invertible as $F_k$ has been chosen to have full row rank.
	    Since layer $k+1$ is fully connected by our assumption,
	    it follows from Definition \ref{def:fc} that 
	    $F_{k+1}=\sigma_{k+1}(F_k W_{k+1} + \ones_Nb_{k+1}^T) = \sigma_{k+1}(F_k W_{k+1}) = D$ 
	    and thus it holds 
	    \begin{align}\label{eq:FAD}
		(F_{k+1})_{i:} = D_{i:} = E_{j:}
	    \end{align}
	    for every $x_i$ from class $j$. 
	    
	    So far, our construction has led to the fact that all training samples belonging to the same class 
	    have identical output at layer $k+1.$
	    The idea now is to see $E$ as a new training data matrix of a subnetwork 
	    consisting of all layers from $k+1$ till the output layer $L$.
	    In particular, layer $k+1$ can be seen as the input layer of this subnetwork 
	    and similarly, layer $L$ can be seen as the output layer.
	    Moreover, every row of $E\in\RR^{m\times n_{k+1}}$ can be seen as a new training sample to this subnetwork.
	    One can see that this subnet together with the training data matrix $E$
	    satisfy the conditions of Theorem \ref{theo:linear_independence} at the last hidden layer $L-1$.
	    In particular, it holds that
	    \begin{itemize}
		\item The rows of $E\in\RR^{m\times n_{k+1}}$ are componentwise different from each other,
		 and thus the input patches must be also different from each other, 
		 and thus $E$ satisfies Assumption \ref{ass:different_patches}
		\item Every layer from $k+1$ till $L-1$ is convolutional or fully connected due to Assumption \ref{ass:architecture}
		\item The width of layer $L-1$ is larger than the number of samples due to Assumption \ref{ass:architecture}, 
		that is, $n_{L-1}\geq n_L=m$
		\item $(\sigma_{k+2},\ldots,\sigma_{L-1})$ satisfy Assumption \ref{ass:activation} due to Assumption \ref{ass:architecture}
	    \end{itemize}
	    By applying Theorem \ref{theo:linear_independence} to this subnetwork and training data $E$,
	    we obtain that there must exist $(W_l,b_l)_{l=k+2}^{L-1}$ 
	    for which all the weight matrices $(U_l)_{l=k+2}^{L-1}$ have full rank
	    such that the set of corresponding $m$ outputs at layer $L-1$ are linearly independent.
	    In particular, let $A\in\RR^{m\times n_{L-1}}$ be the corresponding outputs of $E$ through this subnetwork
	    then it holds that $\rank(A)=m.$
	    Intuitively, if one feeds $E_{j:}$ as an input at layer $k+1$ then 
	    one would get $A_{j:}$ as an output at layer $L-1$.
	    This combined with \eqref{eq:FAD} leads to the fact that if one now feeds $(F_{k+1})_{i:}=E_{j:}$ 
	    as an input at layer $k+1$ then one would get at layer $L-1$ the output 
	    $(F_{L-1})_{i:} = A_{j:}$ whenever $x_i$ belongs to class $j$.
	    
	    Last, we pick $b_L=0, W_L=A^T(AA^T)^{-1} Z.$ 
	    It follows that $AW_L = Z.$
	    Since the output layer $L$ is fully connected, it holds from Definition \ref{def:fc} that 
	    $F_L=F_{L-1}W_L + \ones_Nb_L^T = F_{L-1}W_L.$
	    
	    One can verify now that $F_L=Y.$
	    Indeed, for every sample $x_i$ from class $j$ it holds that
	    \begin{align*}
		(F_L)_{i:} = (F_{L-1})_{i:}^T W_L=A_{j:}^T W_L=Z_{j:}=Y_{i:}.
	    \end{align*}
	    
	    Overall, we have shown that $\Phi=0$. 
	    In addition, it holds $\rank(F_k)=N$ from the construction of the first $k$ layers.
	    All the matrices $(U_l)_{l=k+2}^{L-1}$ have full rank by the construction of the subnetwork from $k+1$ till $L$.
	    Moreover, $U_L=W_L$ also has full rank since $\rank(W_L)=\rank(A^T(AA^T)^{-1}Z)=m.$ 	    
	    Therefore it holds $(W_l,b_l)_{l=1}^L\in  S_k.$
\end{enumerate}

\fi
\end{document}